%% file: main.tex
\documentclass{article}

\usepackage[preprint, nonatbib]{neurips}
\usepackage{natbib}
\setcitestyle{authoryear,round,citesep={;},aysep={,},yysep={;}}

\input{extra_pkgs}

\usepackage{physics}
\usepackage{mathtools}
\usepackage{amsthm}
\usepackage{wrapfig}
\usepackage{subcaption}
\usepackage{multirow}
\usepackage{colortbl}
\usepackage{xspace}

\newtheorem{theorem}{Theorem}
\newtheorem{proposition}{Proposition}
\newtheorem{corollary}{Corollary}
\newtheorem{lemma}{Lemma}
\newtheorem{remark}{Remark}
\theoremstyle{definition}
\newtheorem{definition}{Definition}
\newtcolorbox{theorybox}{colback=gray!8!white,colframe=gray!60!black,boxrule=0.5pt,arc=2pt,left=5pt,right=5pt,top=5pt,bottom=5pt}
\newenvironment{lemma*}[1][]{\begin{theorybox}\noindent\textbf{Lemma}\ifx\\#1\\\else~{[#1]}\fi.~}{\end{theorybox}}
\newenvironment{theorem*}[1][]{\begin{theorybox}\noindent\textbf{Theorem}\ifx\\#1\\\else~{[#1]}\fi.~}{\end{theorybox}}
\newenvironment{proposition*}[1][]{\begin{theorybox}\noindent\textbf{Proposition}\ifx\\#1\\\else~{[#1]}\fi.~}{\end{theorybox}}

\newenvironment{finding}{\begin{tcolorbox}[colback=blue!5!white,colframe=blue!50!black,boxrule=0.5pt,arc=2pt,left=5pt,right=5pt,top=5pt,bottom=5pt]}{\end{tcolorbox}}

\newcommand{\method}{CRISP\xspace}
\newcommand{\methodfull}{\textbf{C}ompressed \textbf{R}easoning via \textbf{I}terative \textbf{S}elf-\textbf{P}olicy Distillation\xspace}

\newcommand{\KL}{\mathrm{KL}}

\newcommand{\EE}{\mathbb{E}}

\newcommand{\D}{\mathcal{D}}
\newcommand{\Lcal}{\mathcal{L}}

\title{CRISP: Compressed Reasoning via Iterative Self-Policy Distillation}

\author{
    Hejian Sang\thanks{Equal contribution.}\thanks{Correspondence to \texttt{hejian@alumni.iastate.edu}} \\
    \texttt{hejian@alumni.iastate.edu}
    \And
    Yuanda Xu\footnotemark[1] \\
    \texttt{yuanda@math.princeton.edu}
    \And
    Zhengze Zhou\footnotemark[1] \\
    \texttt{zz433@cornell.edu}
    \And
    Ran He\footnotemark[1] \\
    \texttt{rh2528@columbia.edu}
    \And
    Zhipeng Wang \\
    \texttt{zhipeng.wang@alumni.rice.edu}
    \And
    Jiachen Sun \\
    \texttt{jiachens@umich.edu}
}

\begin{document}

\maketitle

\begin{abstract}

Reasoning models often generate far more tokens than a task requires, which raises inference cost and can compound errors. We introduce \method (\methodfull), an on-policy self-distillation method that teaches a model to reason more concisely by distilling its own concise behavior back into itself. The method uses a single idea: condition the same model on a ``be concise'' instruction to obtain teacher logits, then minimize the per-token reverse KL divergence between the student and this teacher on the student's own rollouts. It requires no ground-truth answers, no token budgets, and no difficulty estimators. The reverse-KL objective is naturally difficulty-adaptive: it compresses easy problems aggressively while preserving the reasoning steps that hard problems require. On Qwen3-14B, \method cuts reasoning length by up to 56\% on MATH-500 and 38\% on the harder AIME 2024, while improving MATH-500 accuracy by up to 3.3 points over the base model and holding AIME 2024 accuracy within about one point. This behavior generalizes across model sizes and families: Qwen3-8B shows the same compression with accuracy preserved, and DeepSeek-R1-Distill-Llama-8B improves accuracy on all five benchmarks while shortening its responses. General capabilities are preserved across all three models. Code is available at \url{https://github.com/HJSang/OPSD_Reasoning_Compression}.

\end{abstract}

\vspace{1em}
\begin{center}
\includegraphics[width=0.85\textwidth]{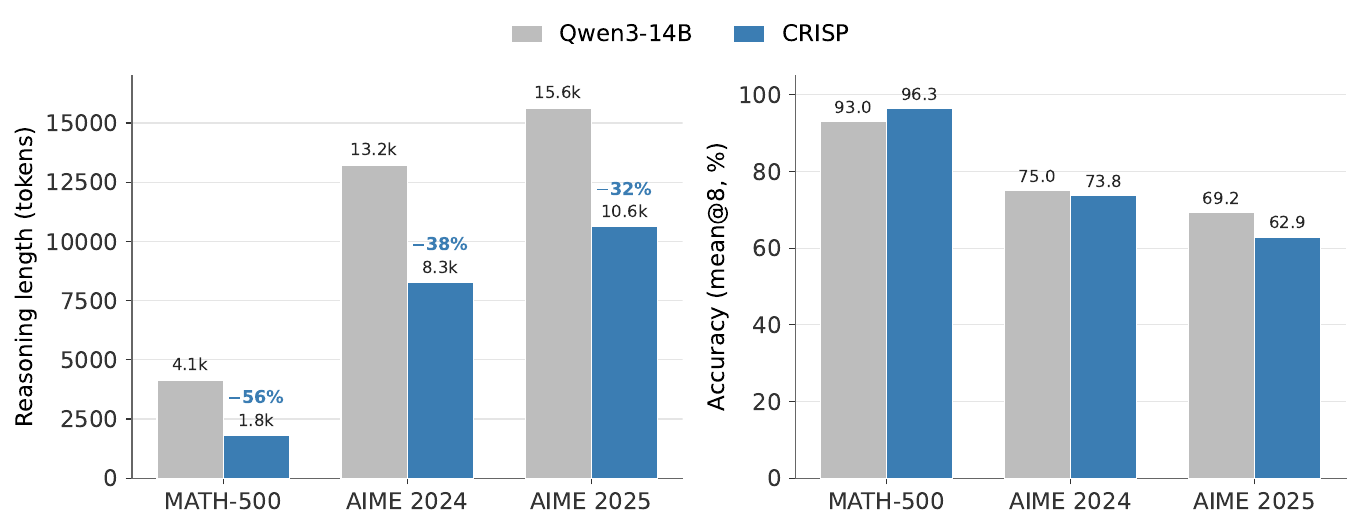}
\captionof{figure}{\textbf{Reasoning compression with preserved accuracy on Qwen3-14B.} Results across three math benchmarks of increasing difficulty under a 30K-token budget (mean@8). \textbf{a}, Average reasoning length: \method reduces length by 56\% on MATH-500, 38\% on AIME 2024, and 32\% on AIME 2025, with the largest reduction on the easiest benchmark. \textbf{b}, Accuracy: \method improves MATH-500 accuracy (93.0 to 96.3) and holds AIME 2024 within about one point, with a modest drop on the hardest AIME 2025 under this aggressive (uniform-instruction) setting.}
\label{fig:hero}
\end{center}
\vspace{1em}

\section{Introduction}
\label{sec:intro}

Large reasoning models generate long chains of intermediate reasoning before producing an answer. Systems such as OpenAI o1~\citep{jaech2024openai}, Gemini 2.5~\citep{comanici2025gemini}, DeepSeek-R1~\citep{deepseek2025r1}, and Qwen3~\citep{qwen2025qwen3} produce thousands of tokens of deliberation, exploring alternatives, checking intermediate steps, and verifying conclusions. This deliberation helps on hard problems, but the models also produce long reasoning traces on easy inputs where a short answer would suffice~\citep{snell2024scaling, muennighoff2025s1}. The excess tokens raise inference cost and latency, and on some problems they degrade accuracy by introducing errors in otherwise correct solutions.

Several families of methods address this problem (Appendix~\ref{app:survey} surveys recent work), but each has a limitation. Reinforcement learning methods add a length penalty to the reward, which requires ground-truth answers and can reduce the model's ability to explore~\citep{aggarwal2025l1, zhang2025dipo, chen2025dler}. Supervised fine-tuning methods train on externally curated short traces, which causes distribution shift and forgetting~\citep{huang2025seer, shenfeld2025sdft}. Most methods apply a uniform compression target rather than adapting to problem difficulty. Prompting methods change behavior only while the prompt is present and do not persist in the weights.

We propose \method (\methodfull), which avoids these limitations with a single mechanism: instruct the model to be concise, then distill that behavior back into the model so it persists without the instruction. Given a reasoning model $\pi_\theta$, we define:
\begin{itemize}[leftmargin=*, itemsep=2pt]
    \item \textbf{Teacher}: $\pi_\theta(\cdot \mid x, c)$, the same model conditioned on a conciseness instruction $c$ (for example: ``Solve concisely, avoid unnecessary steps'').
    \item \textbf{Student}: $\pi_\theta(\cdot \mid x)$, the same model without the compression instruction.
\end{itemize}
Training generates student rollouts and minimizes the per-token reverse KL divergence between student and teacher distributions. This on-policy self-distillation approach requires no ground-truth answers, no reward engineering, and no difficulty estimation. The compression signal emerges naturally from the KL objective, adapting automatically to problem difficulty.

\begin{table}[h]
\centering
\caption{\textbf{Comparison of reasoning compression methods.} \method uniquely combines on-policy training, no dependence on ground-truth (GT) answers, difficulty-adaptive compression, and entropy preservation (Appendix~\ref{app:entropy}).}
\label{tab:positioning}
\small
\renewcommand{\arraystretch}{1.15}
\setlength{\tabcolsep}{2pt}
\begin{tabular}{@{}l c c c c@{}}
\toprule
\textbf{Method} & \makecell{\textbf{On-}\\\textbf{policy}} & \makecell{\textbf{No GT}\\\textbf{needed}} & \makecell{\textbf{Difficulty-}\\\textbf{adaptive}} & \makecell{\textbf{Entropy-}\\\textbf{preserving}} \\
\midrule
RL + length penalty {\small\citep{aggarwal2025l1,zhang2025dipo}} & \cmark & \xmark & \xmark & \xmark \\
SFT on compressed CoT {\small\citep{huang2025seer}} & \xmark & \xmark & \xmark & \cmark \\
OPCD {\small\citep{ye2025opcd}} & \cmark & \xmark & \xmark & \cmark \\
DLER {\small\citep{chen2025dler}} & \cmark & \xmark & \xmark & \xmark \\
Prompting / pruning {\small\citep{xu2025chainofdraft}} & --- & \cmark & \xmark & \cmark \\
\textbf{\method (ours)} & \cmark & \cmark & \cmark & \cmark \\
\bottomrule
\end{tabular}
\end{table}

Table~\ref{tab:positioning} contrasts \method with representative methods from each paradigm. \method is the only approach that satisfies all four desiderata.

\paragraph{Summary of results.} On Qwen3-8B and Qwen3-14B, \method reduces MATH-500 reasoning length by 32\% to 57\% while preserving accuracy, and it improves Qwen3-14B MATH-500 accuracy by up to 3.3 points (Table~\ref{tab:main}). Compression is difficulty-adaptive: reductions are largest on the easier MATH-500 and smaller on the harder AIME benchmarks, where accuracy stays within about one point of the base model. The method readily adopts different concise behaviors: different conciseness instructions all yield strong compression with preserved accuracy, showing that the effect does not depend on a single hand-tuned prompt. The method transfers across model families: on DeepSeek-R1-Distill-Llama-8B it improves accuracy on all five benchmarks while shortening responses. Although \method trains only on math, out-of-domain accuracy on GPQA-Diamond and MMLU is preserved, so compressing math reasoning does not degrade general capabilities.

\section{Related Work}
\label{sec:related}

\paragraph{Reasoning compression via reinforcement learning.}
The most direct approach: penalize length in the reward function. L1~\citep{aggarwal2025l1} caps token count during GRPO training. DiPO~\citep{zhang2025dipo} and DIET~\citep{diet2025} estimate difficulty from rollout pass rates and set per-problem length targets. Leash~\citep{chen2025leash} shapes rewards with sigmoid functions; DLER~\citep{chen2025dler} adds curriculum learning. ThinkPrune~\citep{hou2025thinkprune} continuously trains long-thinking LLMs using reinforcement learning (RL) with an additional token-budget constraint while preserving answer correctness. The catch: all of these require ground-truth answers. No correct answer, no reward and no way to know if compression went too far. ~\citet{xu2026overconfident} further notes that reinforcement learning can induce overconfidence errors, thereby narrowing the model's reasoning boundary and reducing generation diversity.

\paragraph{Reasoning compression via supervised fine-tuning.}
Another route: curate short reasoning traces, then train on them. SEER~\citep{huang2025seer} samples many solutions and keeps the shortest correct ones. TokenSkip~\citep{xia2025tokenskip} learns which tokens to skip. DAP/LiteCoT~\citep{he2025dap} distills from stronger models; S3-CoT~\citep{yang2025s3cot} steers activations toward brevity. The problem is distribution shift: the student trains on someone else's reasoning and forgets its own~\citep{shenfeld2025sdft}.

\paragraph{Training-free compression.}
The lightweight option: change the prompt or the decoder, not the weights. Chain of Draft~\citep{xu2025chainofdraft} asks for minimal drafts instead of full reasoning. TrimR~\citep{wang2025trimr} prunes after the fact. NoWait~\citep{chen2025nowait} and FlowSteer~\citep{flowsteer2025} steer decoding toward conciseness. These methods are easy to deploy but achieve limited compression, and the effect vanishes when you change the prompt.

\paragraph{On-policy self-distillation.}
The closest relatives of our work use the model as its own teacher. OPSD~\citep{zhao2025opsd} gives the teacher the ground-truth answer, achieving 4--8$\times$ efficiency over GRPO. SDPO~\citep{hubotter2025sdpo} conditions on rich feedback for dense credit assignment. SDFT~\citep{shenfeld2025sdft} shows that on-policy distillation dramatically reduces forgetting compared to standard SFT, interpreting it as inverse RL. OPCD~\citep{ye2025opcd} distills system-prompt behaviors into weights. We contribute a new application: using a \emph{conciseness instruction} as the privileged context, achieving compression without any ground-truth supervision.

\section{Method}
\label{sec:method}

\subsection{Problem Formulation}
\label{sec:formulation}

\begin{figure*}[t]
\centering
\begin{tcolorbox}[
    enhanced,
    colback=studentbg,
    colframe=orange!60!black,
    coltitle=orange!60!black,
    colbacktitle=studentbg,
    title={\textbf{Student Prompt}\quad $\pi_\theta(\cdot \mid x)$},
    fonttitle=\bfseries,
    rounded corners,
    boxrule=1.2pt,
    width=0.95\textwidth
]
\ttfamily\small
Solve the following math problem step by step. The last line of your response should be of the form Answer: \$Answer (without quotes) where \$Answer is the answer to the problem.\\[6pt]
Find all real numbers $x$ such that $x^3 - 6x^2 + 11x - 6 = 0$.\\[4pt]
Remember to put your answer on its own line after ``Answer:''.
\end{tcolorbox}
\vspace{6pt}
\begin{tcolorbox}[
    enhanced,
    colback=teacherbg,
    colframe=green!60!black,
    coltitle=green!60!black,
    colbacktitle=teacherbg,
    title={\textbf{Teacher Prompt}\quad $\pi_{\bar{\theta}}(\cdot \mid x{,}\; c)$},
    fonttitle=\bfseries,
    rounded corners,
    boxrule=1.2pt,
    width=0.95\textwidth
]
\ttfamily\small
\textcolor{green!40!black}{\textbf{Conciseness instruction $c$ (uniform compression):}} Solve the following math problem concisely and correctly. Be direct: avoid unnecessary elaboration, redundant steps, or restating the problem. Focus only on the key reasoning steps needed to reach the answer.\\[4pt]
The last line of your response should be of the form Answer: \$Answer (without quotes) where \$Answer is the answer to the problem.\\[6pt]
Find all real numbers $x$ such that $x^3 - 6x^2 + 11x - 6 = 0$.\\[4pt]
Remember to put your answer on its own line after ``Answer:''.
\end{tcolorbox}
\caption{\textbf{Prompt example for student and teacher policies.} Both policies share the same model parameters but differ in conditioning context. The teacher receives only a \emph{conciseness instruction} $c$ prepended to the problem, here the uniform-compression instruction; no ground-truth answers or reference solutions are provided. This is the key distinction from prior self-distillation work~\citep{shenfeld2025sdft}, where the teacher receives the ground-truth solution as privileged information. The student prompt is the original prompt from the DAPO-17K dataset.}
\label{fig:prompts}
\end{figure*}

Consider a reasoning model $\pi_\theta$ that, given input $x$, generates a reasoning trace $r$ followed by an answer $a$, producing output $y = (r, a)$. The reasoning trace typically appears within \texttt{<think>}$\ldots$\texttt{</think>} delimiters. We aim to learn parameters $\theta^*$ such that the model produces shorter reasoning traces while maintaining accuracy.

Let $c$ denote a conciseness instruction (see Figure~\ref{fig:prompts} for a concrete example). Modern reasoning models can follow such instructions via in-context learning, producing shorter reasoning traces when $c$ is prepended to the input. We denote the conciseness-conditioned model as $\pi_\theta(\cdot \mid x, c)$ (teacher) and the unconditional model as $\pi_\theta(\cdot \mid x)$ (student). The teacher and student share parameters $\theta$ but receive different inputs.

\subsection{Training Objective}
\label{sec:objective}

\method minimizes the per-token reverse KL divergence between the student and a stop-gradient teacher on student-generated rollouts:
\begin{equation}
\label{eq:loss}
    \Lcal(\theta) = \EE_{x \sim \D, \; y \sim \pi_\theta(\cdot \mid x)} \left[ \sum_{t=1}^{|y|} D_{\KL} \Big( \pi_\theta(\cdot \mid x, y_{<t}) \;\Big\|\; \pi_{\bar{\theta}}(\cdot \mid x, c, y_{<t}) \Big) \right],
\end{equation}
where $\bar{\theta}$ denotes the teacher weights, which are periodically synchronized with the student (Section~\ref{sec:teacher}), and no gradients flow through the teacher's forward pass. The expectation over $y \sim \pi_\theta(\cdot \mid x)$ makes training \emph{on-policy}: the student is optimized on its own generation distribution, which prevents the distribution shift inherent in off-policy SFT.

\paragraph{Why reverse KL?}
We use reverse KL, $D_\KL(\pi_\theta \| \pi_{\bar\theta})$, rather than the
forward direction $D_\KL(\pi_{\bar\theta} \| \pi_\theta)$, and the choice is
important in our iterative setting. Reverse KL weights each gradient update by
the student's own distribution, so the student only adjusts in token regions it
actually generates. This is mode-seeking: it removes tokens the concise teacher
avoids while leaving the reasoning steps the teacher still uses, and it keeps
updates small because the student already covers the teacher's high-probability
modes. Forward KL instead weights updates by the teacher's distribution,
decoupling the update magnitude from how far the student has drifted. Since our
teacher is a stale copy of the student rather than a fixed external model, this
is unstable in practice (Appendix~\ref{app:kl_direction}): on Qwen3-8B, forward
KL collapses within the first ${\sim}100$ steps, with accuracy falling to near
zero and response length diverging to the token budget, whereas reverse KL
trains stably.

\subsection{Teacher Parameterization}
\label{sec:teacher}

A natural baseline is a \emph{fully frozen} teacher ($\bar{\theta} = \theta_0$) as in \citet{zhao2025opsd}. While simple and stable, the frozen teacher becomes an increasingly weak compression target as the student improves: once the student has internalized the initial conciseness signal, no further compression is possible because the reference distribution no longer leads the student.

To address this, we adopt a \emph{periodic teacher update} strategy. The teacher weights are synchronized with the current student weights every $M$ training steps:
\begin{equation}
\label{eq:teacher_update}
    \bar{\theta} \leftarrow \theta \quad \text{every } M \text{ steps}.
\end{equation}
Each refresh creates a new, stronger compression target: the updated teacher, when conditioned on the conciseness instruction $c$, produces traces that are more concise than the previous teacher's (since the student, now serving as the new teacher, has already learned to compress). This \emph{progressive compression} effect pushes the student to continuously shorten its reasoning over the course of training, beyond what a single frozen reference can achieve.

\paragraph{Difficulty-adaptive compression.}
\label{sec:difficulty}
Compression adapts naturally to problem difficulty: for easy problems, the concise teacher produces much shorter traces, creating strong KL signal; for hard problems, even the teacher needs extensive reasoning, yielding weak signal. We formalize this in Proposition~\ref{prop:difficulty} and verify it empirically in Section~\ref{sec:difficulty_analysis}.

\subsection{Training Algorithm}
\label{sec:algorithm}

The complete \method training procedure is given in Algorithm~\ref{alg:opsdc}.

\begin{algorithm}[t]
\caption{\method: On-Policy Self-Distillation for Concise Reasoning}
\label{alg:opsdc}
\KwIn{Model $\pi_\theta$, dataset $\D = \{x_i\}$, conciseness instruction $c$, learning rate $\eta$, teacher update interval $M$}
\KwOut{Compressed reasoning model $\pi_{\theta^*}$}
Initialize teacher: $\bar{\theta} \leftarrow \theta_0$\;
\For{each training step $k = 1, 2, \ldots$}{
    \If{$k \bmod M = 0$}{
        Update teacher: $\bar{\theta} \leftarrow \theta$ \tcp*{periodic refresh}
    }
    Sample batch $\{x_1, \ldots, x_B\} \sim \D$\;
    \For{each $x_i$ in batch}{
        Generate student rollout: $y_i \sim \pi_\theta(\cdot \mid x_i)$\; \label{line:rollout}
        \For{each token position $t = 1, \ldots, |y_i|$}{
            Compute student logits: $q_t \leftarrow \pi_\theta(\cdot \mid x_i, y_{i,<t})$\;
            Compute teacher logits: $p_t \leftarrow \pi_{\bar{\theta}}(\cdot \mid x_i, c, y_{i,<t})$ \tcp*{no grad}
            Compute $D_\KL(q_t \| p_t)$\;
        }
        $\Lcal_i \leftarrow \sum_{t} D_\KL(q_t \| p_t)$\;
    }
    Update student: $\theta \leftarrow \theta - \eta \nabla_\theta \frac{1}{B}\sum_i \Lcal_i$\;\tcp*{normalized by $|y_i|$ in practice}
}
\Return $\pi_{\theta^*}$\;
\end{algorithm}

\paragraph{Computational cost and simplicity.} The entire training pipeline requires only standard supervised training infrastructure: no reward models, no value functions, no advantage estimation, and no multi-rollout sampling. Each training step requires two forward passes per rollout token: one for the student (with gradient) and one for the teacher (without gradient, and cacheable within each $M$-step window). The periodic teacher refresh (Eq.~\ref{eq:teacher_update}) is a simple weight copy with negligible cost. This simplicity yields substantial efficiency gains over RL methods, which require multiple rollouts per prompt, reward model inference, and complex optimization (e.g., PPO clipping, GAE).

The per-token KL objective also does not require complete rollouts, because it supplies a training signal at every position rather than only at the final answer. Truncated rollouts therefore suffice: our ablation (Section~\ref{sec:ablation_len}) shows that \method trains robustly across maximum rollout lengths of 1K, 4K, 8K, and 30K tokens, with comparable compression and accuracy. Short rollouts cut generation cost, the dominant expense in on-policy training, without degrading the result. This is a further advantage over outcome-reward RL, which needs full-length completions to compute a terminal reward.





\section{Theoretical Analysis}
\label{sec:theory_summary}

We now summarize key theoretical properties of \method that illuminate why such a simple objective can produce strong compression without the failure modes of length-penalized RL. In particular, we connect the per-token loss to sequence-level KL, interpret the update as implicit reward maximization, and analyze when compression preserves accuracy, adapts to difficulty, and avoids catastrophic forgetting. Proof sketches are provided inline; full proofs are deferred to Appendix~\ref{sec:theory_main}.

\subsection{Training Loss as Sequence-Level KL}
\label{sec:theory_chain}

The first result connects the practical per-token training objective to a standard information-theoretic quantity, enabling all subsequent analysis.

\begin{lemma*}[Training loss equals sequence-level KL; Lemma~\ref{lem:chain_rule}]
The per-token \method loss \textnormal{(Eq.~\ref{eq:loss})} equals the sequence-level KL divergence $D_\KL(\pi_\theta(\cdot \mid x) \| \pi_{\bar{\theta}}(\cdot \mid x, c))$ by the chain rule for autoregressive models.
\end{lemma*}

\begin{proof}[Proof sketch]
By the autoregressive factorization $q(y \mid x) = \prod_t q(y_t \mid x, y_{<t})$, the log-ratio $\log \frac{q(y \mid x)}{p(y \mid x)}$ decomposes into $\sum_t \log \frac{q(y_t \mid x, y_{<t})}{p(y_t \mid x, y_{<t})}$. Taking expectations over $y \sim q$ yields the per-token KL sum, which equals the sequence-level KL by definition.
\end{proof}

This identification underpins all subsequent results by letting us apply standard information-theoretic tools (Pinsker's inequality, the data-processing inequality) to the per-token loss.

\subsection{Implicit Reward Interpretation}
\label{sec:theory}

Following the inverse RL framework of~\citet{shenfeld2025sdft}, we show that \method implicitly maximizes a reward function that combines task performance with a conciseness preference.

\begin{theorem}[Implicit reward]
\label{thm:reward}
The \method objective (Eq.~\ref{eq:loss}) is equivalent to maximizing the expected implicit reward:
\begin{equation}
    r(y_t, x) = \log \pi_{\bar{\theta}}(y_t \mid x, c, y_{<t}) - \log \pi_\theta(y_t \mid x, y_{<t}).
\end{equation}
\end{theorem}

\begin{proof}[Proof sketch]
Expanding the reverse KL:
\begin{align}
    D_\KL\big(\pi_\theta(\cdot \mid x, y_{<t}) \| \pi_{\bar{\theta}}(\cdot \mid x, c, y_{<t})\big) &= \EE_{y_t \sim \pi_\theta} \left[\log \frac{\pi_\theta(y_t \mid x, y_{<t})}{\pi_{\bar{\theta}}(y_t \mid x, c, y_{<t})}\right] \\
    &= -\EE_{y_t \sim \pi_\theta}\big[r(y_t, x)\big].
\end{align}
Since $r(y_t, x) = \log \pi_{\bar{\theta}} - \log \pi_\theta$ naturally decomposes into a teacher-favoring term $\log \pi_{\bar{\theta}}$ and an entropy-like term $-\log \pi_\theta$, minimizing reverse KL can be interpreted as maximizing an implicit, policy-dependent reward-shaping objective. This is closely related to maximum-entropy RL, but not identical to the standard setting with a fixed environment reward.
\end{proof}

\begin{remark}
The implicit reward $r(y_t, x)$ is positive when the concise teacher assigns higher probability to token $y_t$ than the student does, and negative otherwise. Thus, \method implicitly rewards concise reasoning without any explicit length penalty. Within each $M$-step window, the teacher weights $\bar{\theta}$ are fixed, serving as an implicit trust region: the student can only move as far as the teacher's concise distribution allows, preventing unbounded policy drift. The periodic refresh (Eq.~\ref{eq:teacher_update}) then shifts this trust region forward, enabling progressive compression across windows.
\end{remark}

\subsection{Accuracy Preservation}
\label{sec:theory_accuracy}

A natural concern is whether compression degrades accuracy. The following theorem shows that accuracy loss is bounded by two interpretable quantities.

\begin{theorem*}[Accuracy preservation; Theorem~\ref{thm:accuracy}]
If training converges to loss $\epsilon_\KL$ and the concise teacher preserves accuracy to within $\epsilon_T$ of the base model, the student satisfies:
\begin{equation}
\mathrm{Acc}(\pi_{\theta^*}) \geq \mathrm{Acc}(\pi_{\bar{\theta}}) - \epsilon_T - \sqrt{\epsilon_\KL / 2}.
\end{equation}
\end{theorem*}

\begin{proof}[Proof sketch]
Apply Pinsker's inequality to convert the KL bound $\epsilon_\KL$ into a total variation bound $\sqrt{\epsilon_\KL/2}$ between student and teacher distributions. Total variation bounds the difference in probability of any event, in particular the correctness event $A(x)$, so the student's accuracy is within $\sqrt{\epsilon_\KL/2}$ of the teacher's. Combining with the teacher quality assumption $\epsilon_T$ via the triangle inequality yields the result.
\end{proof}

The bound decomposes accuracy loss into two independent, interpretable terms: teacher quality ($\epsilon_T$) and distillation gap ($\sqrt{\epsilon_\KL/2}$). When the concise teacher is at least as accurate as the base model ($\epsilon_T \leq 0$), the bound becomes $\mathrm{Acc}(\pi_{\theta^*}) \geq \mathrm{Acc}(\pi_{\bar{\theta}}) + |\epsilon_T| - \sqrt{\epsilon_\KL/2}$, so accuracy is preserved, and improves whenever the teacher's gain exceeds the distillation gap. This matches Table~\ref{tab:main}: accuracy holds within about a point on the already-strong Qwen3 models and improves where the base model has headroom (e.g.\ DeepSeek-R1-Distill-Llama-8B).

\subsection{Difficulty-Adaptive Compression}
\label{sec:theory_difficulty}

A key design question for any compression method is how to allocate budget across problems of varying difficulty. We show that \method handles this \emph{automatically}: the compression signal is provably stronger on easy problems.

\begin{proposition*}[Difficulty-adaptive compression; Proposition~\ref{prop:difficulty}]
The per-token compression signal $S(x)$ is non-increasing in problem difficulty $d(x)$: easy problems receive strong compression pressure while hard problems, whose reasoning steps are predominantly essential, receive weak pressure.
\end{proposition*}

\begin{proof}[Proof sketch]
Decompose the normalized KL into essential and compressible token contributions: $S(x) = \rho(x) \cdot D_\mathcal{E} + (1 - \rho(x)) \cdot D_\mathcal{C}$, where $\rho(x)$ is the fraction of essential tokens, and $D_\mathcal{E}, D_\mathcal{C}$ are the category-level KL divergences. Since compressible tokens carry strictly larger KL ($D_\mathcal{C} > D_\mathcal{E}$) and the essential fraction $\rho(x)$ is non-decreasing in difficulty, $S(x) = D_\mathcal{C} - \rho(x)(D_\mathcal{C} - D_\mathcal{E})$ is a decreasing affine function of $\rho(x)$, hence non-increasing in $d(x)$.
\end{proof}

This formalizes the empirical pattern (Table~\ref{tab:main}) that \method compresses the easier MATH-500 more aggressively than the harder AIME benchmarks: on Qwen3-14B, reductions are up to 56\% on MATH-500 but only 32--38\% on AIME, without any explicit difficulty estimation.

\subsection{Bounded Forgetting}
\label{sec:theory_forgetting}

A central advantage of on-policy self-distillation over off-policy SFT is controlled divergence from the original model. We formalize this via the \emph{conciseness gap}.

\begin{proposition*}[Bounded forgetting; Proposition~\ref{prop:forgetting}]
Divergence from the base model is bounded by:
\begin{equation}
\EE_x\big[d_\mathrm{TV}(\pi_{\theta^*}(\cdot \mid x),\, \pi_{\theta_0}(\cdot \mid x))\big] \leq \sqrt{\epsilon_\KL / 2} + \EE_x[\gamma(x)],
\end{equation}
where $\gamma(x)$ is the conciseness gap, i.e., the total variation between the base model's outputs with and without the conciseness instruction.
\end{proposition*}

\begin{proof}[Proof sketch]
Apply the triangle inequality for total variation: $d_\mathrm{TV}(\pi_{\theta^*}, \pi_{\theta_0}) \leq d_\mathrm{TV}(\pi_{\theta^*}, \pi_{\theta_0}(\cdot \mid c)) + d_\mathrm{TV}(\pi_{\theta_0}(\cdot \mid c), \pi_{\theta_0})$. The first term is bounded by $\sqrt{\epsilon_\KL/2}$ via Pinsker's inequality on the converged training loss; the second term is the conciseness gap $\gamma(x)$ by definition.
\end{proof}

For hard problems where the conciseness instruction has little effect, $\gamma(x) \approx 0$, so forgetting is minimal where it matters most. This contrasts with off-policy SFT, whose forgetting depends on the full distribution mismatch between teacher data and the base model, a gap that can be arbitrarily large.

\subsection{Compression Reduces Compounding Error}
\label{sec:theory_compounding}

Finally, we provide a probabilistic model explaining the most striking empirical finding: shorter reasoning traces can \emph{improve} accuracy rather than degrade it.

\begin{proposition*}[Compression reduces compounding error; Proposition~\ref{prop:compounding}]
Under a model where each token independently introduces a reasoning error with probability $p_\mathrm{err}$, compressing from $L$ to $\alpha L$ tokens yields an accuracy ratio $(1 - p_\mathrm{err})^{-(1-\alpha)L}$, which grows exponentially in the number of removed tokens.
\end{proposition*}

\begin{proof}[Proof sketch]
Direct computation: the accuracy ratio $(1-p_\mathrm{err})^{\alpha L} / (1-p_\mathrm{err})^L = (1-p_\mathrm{err})^{-(1-\alpha)L}$. Using $\ln(1-p) \leq -p$ and $e^u \geq 1 + u$, this is at least $1 + (1-\alpha)L \cdot p_\mathrm{err}$. On MATH-500 with $L \approx 4{,}139$ and $\alpha \approx 0.44$ (Qwen3-14B, 30K budget; Table~\ref{tab:main}), even $p_\mathrm{err} = 10^{-4}$ gives a linear lower bound of ${\sim}23\%$ relative accuracy improvement (${\sim}26\%$ from the exact exponential form).
\end{proof}

This provides a \emph{lower} bound on the accuracy benefit of compression: in practice, reasoning errors are positively correlated (one incorrect step causes subsequent steps to build on a false premise), amplifying the gain beyond the independence assumption. Consistent with this, compression improves accuracy where the base model has room to improve, most clearly on DeepSeek-R1-Distill-Llama-8B (MATH-500 $71.3{\to}82.1$; Table~\ref{tab:main}), while preserving accuracy on the already-strong Qwen3 models.

\section{Experiments}
\label{sec:experiments}
\subsection{Experimental Setting}
\label{sec:setup}

\paragraph{Models and data.} We evaluate \method on Qwen3-8B and Qwen3-14B~\citep{qwen2025qwen3} and DeepSeek-R1-Distill-Llama-8B~\citep{deepseek2025r1}, training on ${\sim}$13,600 competition-level math problems from DAPO-Math-17k~\citep{yu2025dapo} \emph{without ground-truth answers}; only problem statements are used to generate student rollouts. We train for 1 epoch with learning rate $1 \times 10^{-6}$, global batch size 32, periodic teacher update (interval $M{=}50$; see ablation in Section~\ref{sec:ablation_M}), and $8\times$ H200 GPUs. Although nominally a full epoch, the algorithm converges quickly at around ${\sim}$100 steps. Each prompt generates a single student rollout (temperature 1.0) with a maximum response length of 8,192 tokens. Because \method optimizes a per-token KL objective rather than an outcome-based reward, there is no need to generate complete responses as in RL methods; partial rollouts already provide a useful training signal. This phenomenon is also observed by \cite{chen2025distilling} in the context of knowledge distillation. Full training and infrastructure details are in Appendix~\ref{app:implementation}.

\paragraph{Conciseness instructions.} The teacher is defined by the conciseness instruction $c$ prepended to the problem. We use two instructions, shown below: the \emph{uniform concise prompt} (v1), which asks for directness on every problem, and the \emph{difficulty-aware concise prompt} (v2), which additionally instructs the model not to over-compress hard problems. Both are prepended only to the teacher; the student sees the plain problem.

\begin{tcolorbox}[enhanced, colback=teacherbg, colframe=green!55!black, boxrule=0.8pt, left=3pt, right=3pt, top=2pt, bottom=2pt, fonttitle=\bfseries, title={Conciseness instructions $c$}]
\small
\textbf{v1 (uniform concise prompt):} Solve the following math problem concisely and correctly. Be direct: avoid unnecessary elaboration, redundant steps, or restating the problem. Focus only on the key reasoning steps needed to reach the answer.\\[4pt]
\textbf{v2 (difficulty-aware concise prompt):} Solve the following math problem concisely and correctly. Prioritize a correct final answer above brevity. For straightforward problems, be direct: avoid unnecessary elaboration, redundant steps, or restating the problem. For difficult or multi-step problems, do NOT over-compress: keep every reasoning step needed for correctness, including case analysis, edge cases, and a brief check of the final answer. Never omit a step that the answer depends on.
\end{tcolorbox}

\paragraph{Benchmarks.} \method is trained only on math prompts, so we separate evaluation into in-domain and out-of-domain tasks. \emph{In-domain}, we use three mathematical reasoning benchmarks spanning a wide difficulty range: MATH-500~\citep{hendrycks2021math} (500 problems), AIME 2024 (30 problems), and AIME 2025 (30 problems). \emph{Out-of-domain}, we use two general-capability benchmarks that the model never trains on: GPQA-Diamond~\citep{rein2023gpqa} (graduate-level, Google-proof science questions) and MMLU~\citep{hendrycks2021mmlu} (massive multitask knowledge). The out-of-domain benchmarks test whether compressing math reasoning degrades general knowledge and scientific reasoning. We define a \emph{token budget} as the maximum response length allowed during inference, a practical lever for controlling serving cost, and report results under a 30,000-token budget, which effectively eliminates truncation and enables a fair accuracy comparison.

\paragraph{Scoring.} Unless otherwise noted, all math accuracies reported in this paper (MATH-500, AIME 2024/2025) use the \emph{dual-path} scorer: a response is correct if \emph{either} an extracted ``Answer: $X$'' line or a literal ``$\backslash$boxed\{$\cdot$\}'' expression math-verifies against the gold answer, building on the veRL math grading utility~\citep{sheng2024hybridflow}.\footnote{Dual-path scoring: \url{https://github.com/HJSang/CRISP_Reasoning_Compression/blob/main/workspace/src/rewards/dual_path_math_verify.py}, built on veRL's \texttt{math\_dapo} grader: \url{https://github.com/verl-project/verl/blob/main/verl/utils/reward_score/math_dapo.py}.} This avoids undercounting models whose native answer format differs from the boxed convention (Appendix~\ref{app:answer_format}). GPQA-Diamond and MMLU are scored by exact letter match on the multiple-choice answer.

\subsection{Main Results}
\label{sec:main_results}

\begin{table}[t]
\centering
\caption{\textbf{\method compresses reasoning while preserving accuracy (token budget = 30K).} For each benchmark we report accuracy (Acc, mean@8, \%) and token reduction relative to the base model (Red., \%; ``---'' marks the base reference, negative values indicate the response grew longer). \method is trained with a periodic teacher update ($M{=}50$); v1 and v2 denote the uniform and difficulty-aware conciseness instructions (Section~\ref{sec:ablation_prompt}). The first three benchmarks are in-domain; GPQA-Diamond and MMLU are out-of-domain and verify that general capabilities are preserved. Per-row response lengths and the inference-only ``concise prompt'' baselines are in Appendix~\ref{app:full_main}, Table~\ref{tab:main_full}.}
\label{tab:main}
\small
\setlength{\tabcolsep}{4pt}
\resizebox{\linewidth}{!}{%
\begin{tabular}{@{}l cc cc cc cc cc@{}}
\toprule
& \multicolumn{2}{c}{\textbf{MATH-500}} & \multicolumn{2}{c}{\textbf{AIME 2024}} & \multicolumn{2}{c}{\textbf{AIME 2025}} & \multicolumn{2}{c}{\textbf{GPQA-D}} & \multicolumn{2}{c}{\textbf{MMLU}} \\
\cmidrule(lr){2-3} \cmidrule(lr){4-5} \cmidrule(lr){6-7} \cmidrule(lr){8-9} \cmidrule(lr){10-11}
\textbf{Method} & Acc & Red. & Acc & Red. & Acc & Red. & Acc & Red. & Acc & Red. \\
\midrule
\multicolumn{11}{l}{\textit{Qwen3-8B}} \\
Base Model   & 95.7 & --- & \textbf{76.2} & --- & \textbf{70.4} & --- & \textbf{61.5} & --- & 81.9 & --- \\
\method~(v2) & \textbf{95.7} & 31.6\% & 75.0 & 17.1\% & 65.8 & 17.5\% & 58.3 & 17.2\% & 81.2 & 22.4\% \\
\method~(v1) & \textbf{95.7} & \textbf{56.9\%} & 72.9 & \textbf{32.9\%} & 58.8 & \textbf{28.4\%} & 58.5 & \textbf{36.2\%} & 80.9 & \textbf{44.7\%} \\
\midrule
\multicolumn{11}{l}{\textit{Qwen3-14B}} \\
Base Model   & 93.0 & --- & 75.0 & --- & \textbf{69.2} & --- & \textbf{62.2} & --- & \textbf{85.1} & --- \\
\method~(v2) & 95.2 & 34.7\% & \textbf{75.0} & 19.7\% & 67.1 & 16.8\% & 62.0 & 20.7\% & 83.9 & 22.4\% \\
\method~(v1) & \textbf{96.3} & \textbf{56.3\%} & 73.8 & \textbf{37.5\%} & 62.9 & \textbf{32.1\%} & 61.9 & \textbf{39.7\%} & 84.2 & \textbf{43.1\%} \\
\midrule
\multicolumn{11}{l}{\textit{DeepSeek-R1-Distill-Llama-8B}} \\
Base Model   & 71.3 & --- & 33.3 & --- & 25.0 & --- & 47.0 & --- & 71.5 & --- \\
\method~(v2) & 79.8 & 23.2\% & \textbf{42.1} & $-$2.5\% & 26.2 & 0.1\% & 46.7 & 7.0\% & 71.4 & 11.4\% \\
\method~(v1) & \textbf{82.1} & \textbf{31.6\%} & 39.2 & \textbf{6.3\%} & \textbf{27.1} & \textbf{7.1\%} & \textbf{48.3} & \textbf{10.2\%} & 71.7 & \textbf{17.6\%} \\
\bottomrule
\end{tabular}%
}
\end{table}

Table~\ref{tab:main} presents our main results under the 30,000-token budget, which eliminates response truncation for a fair accuracy comparison.\footnote{MMLU is evaluated using the Language Model Evaluation Harness~\citep{eval-harness}: \url{https://github.com/EleutherAI/lm-evaluation-harness}.} The first three benchmarks (MATH-500, AIME 2024/2025) are in-domain, and the last two (GPQA-Diamond, MMLU) are out-of-domain: \method trains only on math, so the out-of-domain columns measure whether compression harms general capabilities. We highlight four findings.

\paragraph{\method compresses reasoning while preserving accuracy.}
Across all three models, \method substantially shortens reasoning traces while keeping accuracy close to the base model, and it improves accuracy where the base model has room. On the in-domain math benchmarks it reduces length by up to 57\% (MATH-500) with accuracy held within about one point on MATH-500 and the AIME benchmarks. The accuracy effect depends on base-model headroom: the already-strong Qwen3-8B and Qwen3-14B (both above 93\% on MATH-500) are preserved, with Qwen3-14B gaining up to 3.3 points, while DeepSeek-R1-Distill-Llama-8B starts lower and \emph{improves} on all five benchmarks, by up to 10.8 points on MATH-500. The concise teacher removes redundant tokens that would otherwise accumulate errors, so the accuracy benefit is largest where the base model leaves the most room.

\paragraph{Compression naturally adapts to problem difficulty.}
\label{sec:difficulty_analysis}
Prior RL methods estimate difficulty from rollout pass rates~\citep{zhang2025dipo, diet2025} or train separate difficulty classifiers; \method requires none of this, as difficulty adaptation emerges from the KL objective. Using benchmarks as a difficulty proxy, \method compresses the easier MATH-500 by 32\% to 57\% but the harder AIME benchmarks by only 5\% to 38\%. The model compresses easy problems more than hard ones with no explicit difficulty estimate, which follows from the per-token KL signal (see \S\ref{sec:difficulty}).

\paragraph{Different conciseness instructions are readily adopted.}
\method works with different concise behaviors rather than a single hand-tuned prompt. The uniform instruction (v1) compresses more aggressively (up to 57\% on MATH-500) but drops more accuracy on the hardest benchmark, AIME 2025 (e.g.\ $-11.6$ points on Qwen3-8B). The difficulty-aware instruction (v2) compresses less (around 32\% to 35\% on MATH-500) but protects hard problems, cutting the AIME 2025 drop to $-4.6$ points on Qwen3-8B and $-2.1$ on Qwen3-14B. Both give strong compression with preserved accuracy, so the wording of the instruction trades compression against accuracy on hard problems without changing the overall behavior.

\paragraph{General knowledge and capabilities are preserved.}
Although \method trains only on math, out-of-domain accuracy is preserved: on GPQA-Diamond and MMLU, every model stays within about one point of its base accuracy, and DeepSeek even improves slightly. Compressing math reasoning therefore does not degrade general knowledge or scientific reasoning.

\paragraph{Qualitative examples.}
\label{subsec:wrong_to_correct}
Figure~\ref{fig:wrong_to_correct} shows two examples of \method compressing reasoning while reaching the same correct answer as the base model, and honoring the two conciseness instructions. In each case the base model, the difficulty-aware teacher (v2), and the uniform teacher (v1) all answer correctly, but v2 is much shorter than the base and v1 is shorter still, illustrating that v2 keeps slightly more reasoning than v1 (more examples in Appendix~\ref{app:qualitative}).

\begin{figure}[h!]
\centering

\textbf{Problem 1 (MATH-500):} \emph{Twelve $1\times1$ squares form a $3\times4$ rectangle; two shaded right triangles are drawn. What is the total shaded area?} \quad (Correct answer: 10)

\vspace{0.4em}
\begin{subfigure}[t]{0.32\textwidth}
\begin{tcolorbox}[enhanced, colback=gray!5, colframe=gray!50, fonttitle=\bfseries\scriptsize, title={Qwen3-8B base (8{,}422 tokens)}, left=3pt, right=3pt, top=2pt, bottom=2pt]
\scriptsize
{\ttfamily\color{gray!60}<think>}

\textcolor{gray}{\ldots}re-derives the $3\times4$ dimensions from the Asymptote code, then re-visualizes each triangle and repeatedly second-guesses which segment is the base\ldots

\textcolor{blue}{[several restatements of the same two triangles before committing]}

{\ttfamily\color{gray!60}</think>}

\vspace{1pt}\noindent\textcolor{gray!30}{\rule{\linewidth}{0.3pt}}\vspace{2pt}

Triangle 1: $\tfrac12\cdot2\cdot4=4$. Triangle 2: $\tfrac12\cdot3\cdot4=6$. Total $=10$.

Answer: \textcolor{green!50!black}{$\boldsymbol{10}$} \rlap{\;\textcolor{green!50!black}{\checkmark}}
\end{tcolorbox}
\end{subfigure}
\hfill
\begin{subfigure}[t]{0.32\textwidth}
\begin{tcolorbox}[enhanced, colback=teacherbg, colframe=green!55!black, fonttitle=\bfseries\scriptsize, title={\method~v2 (2{,}466 tok, \textbf{71\%})}, left=3pt, right=3pt, top=2pt, bottom=2pt]
\scriptsize
{\ttfamily\color{gray!60}<think>}

The grid is $3\times4$. Triangle 1 has legs $2$ and $4$: area $\tfrac12\cdot2\cdot4=4$. Triangle 2 has legs $3$ and $4$: area $\tfrac12\cdot3\cdot4=6$. The two triangles do not overlap (left vs.\ right), so the areas add.

{\ttfamily\color{gray!60}</think>}

\vspace{1pt}\noindent\textcolor{gray!30}{\rule{\linewidth}{0.3pt}}\vspace{2pt}

Total shaded area $=4+6=10$.

Answer: \textcolor{green!50!black}{$\boldsymbol{10}$} \rlap{\;\textcolor{green!50!black}{\checkmark}}
\end{tcolorbox}
\end{subfigure}
\hfill
\begin{subfigure}[t]{0.32\textwidth}
\begin{tcolorbox}[enhanced, colback=orange!3, colframe=orange!60!black, fonttitle=\bfseries\scriptsize, title={\method~v1 (2{,}018 tok, \textbf{76\%})}, left=3pt, right=3pt, top=2pt, bottom=2pt]
\scriptsize
{\ttfamily\color{gray!60}<think>}

Total rectangle area is $12$. Two right triangles: base $2$, height $4$ gives $4$; base $3$, height $4$ gives $6$.

{\ttfamily\color{gray!60}</think>}

\vspace{1pt}\noindent\textcolor{gray!30}{\rule{\linewidth}{0.3pt}}\vspace{2pt}

Areas: $4$ and $6$, sum $=10$.

Answer: \textcolor{green!50!black}{$\boldsymbol{10}$} \rlap{\;\textcolor{green!50!black}{\checkmark}}
\end{tcolorbox}
\end{subfigure}

\vspace{1.2em}

\textbf{Problem 2 (AIME 2024):} \emph{On an $8\times8$ grid, count length-16 lattice paths from the lower-left to the upper-right corner that change direction exactly four times.} \quad (Correct answer: 294)

\vspace{0.4em}
\begin{subfigure}[t]{0.32\textwidth}
\begin{tcolorbox}[enhanced, colback=gray!5, colframe=gray!50, fonttitle=\bfseries\scriptsize, title={Qwen3-8B base (6{,}041 tokens)}, left=3pt, right=3pt, top=2pt, bottom=2pt]
\scriptsize
{\ttfamily\color{gray!60}<think>}

\textcolor{gray}{\ldots}re-establishes that a path is $8$ R and $8$ U moves, then works through the run structure slowly, re-checking the ``four direction changes'' condition several times\ldots

{\ttfamily\color{gray!60}</think>}

\vspace{1pt}\noindent\textcolor{gray!30}{\rule{\linewidth}{0.3pt}}\vspace{2pt}

Start U: $\binom{7}{2}\binom{7}{1}=147$; start R: $147$. Total $294$.

Answer: \textcolor{green!50!black}{$\boldsymbol{294}$} \rlap{\;\textcolor{green!50!black}{\checkmark}}
\end{tcolorbox}
\end{subfigure}
\hfill
\begin{subfigure}[t]{0.32\textwidth}
\begin{tcolorbox}[enhanced, colback=teacherbg, colframe=green!55!black, fonttitle=\bfseries\scriptsize, title={\method~v2 (3{,}213 tok, \textbf{47\%})}, left=3pt, right=3pt, top=2pt, bottom=2pt]
\scriptsize
{\ttfamily\color{gray!60}<think>}

A path is $8$ R and $8$ U. Four direction changes means five monotone runs, so the moves split as $3$ runs of one letter and $2$ of the other. Distribute $8$ U's into $3$ runs: $\binom{7}{2}$; $8$ R's into $2$ runs: $\binom{7}{1}$. Two cases by starting letter.

{\ttfamily\color{gray!60}</think>}

\vspace{1pt}\noindent\textcolor{gray!30}{\rule{\linewidth}{0.3pt}}\vspace{2pt}

$2\times\binom{7}{2}\binom{7}{1}=2\cdot21\cdot7=294$.

Answer: \textcolor{green!50!black}{$\boldsymbol{294}$} \rlap{\;\textcolor{green!50!black}{\checkmark}}
\end{tcolorbox}
\end{subfigure}
\hfill
\begin{subfigure}[t]{0.32\textwidth}
\begin{tcolorbox}[enhanced, colback=orange!3, colframe=orange!60!black, fonttitle=\bfseries\scriptsize, title={\method~v1 (1{,}743 tok, \textbf{71\%})}, left=3pt, right=3pt, top=2pt, bottom=2pt]
\scriptsize
{\ttfamily\color{gray!60}<think>}

$8$ R and $8$ U moves; four direction changes give five runs, partitioned $3{+}2$. Segment counts: $\binom{7}{2}$ and $\binom{7}{1}$, times two starting letters.

{\ttfamily\color{gray!60}</think>}

\vspace{1pt}\noindent\textcolor{gray!30}{\rule{\linewidth}{0.3pt}}\vspace{2pt}

$2\times21\times7=294$.

Answer: \textcolor{green!50!black}{$\boldsymbol{294}$} \rlap{\;\textcolor{green!50!black}{\checkmark}}
\end{tcolorbox}
\end{subfigure}

\vspace{1.2em}

\textbf{Problem 3 (AIME 2025, format change):} \emph{Find $m+n$ where $\tfrac{m}{n}$ (in lowest terms) is the sum of all $k$ for which $|25+20i-z|=5$ and $|z-4-k|=|z-3i-k|$ has exactly one solution $z$.} \quad (Correct answer: 77)

\vspace{0.4em}
\begin{subfigure}[t]{0.32\textwidth}
\begin{tcolorbox}[enhanced, colback=gray!5, colframe=gray!50, fonttitle=\bfseries\scriptsize, title={Qwen3-8B base (4{,}286 tokens)}, left=3pt, right=3pt, top=2pt, bottom=2pt]
\scriptsize
{\ttfamily\color{gray!60}<think>}

\textcolor{gray}{\ldots}the circle is tangent to the perpendicular bisector line; solving the tangency condition gives $k=\tfrac{23}{8}$ and $k=\tfrac{123}{8}$, so the sum is $\tfrac{73}{4}$\ldots

{\ttfamily\color{gray!60}</think>}

\vspace{1pt}\noindent\textcolor{gray!30}{\rule{\linewidth}{0.3pt}}\vspace{2pt}

$m=73,\,n=4$, so $m+n = \boxed{77}$.

\vspace{2pt}
\textcolor{red}{\scriptsize No ``Answer:'' line; result only inside \texttt{\textbackslash boxed\{\}}.} \rlap{\;\textcolor{red}{\ding{55}}\;\scriptsize\textcolor{red}{(strict format)}}
\end{tcolorbox}
\end{subfigure}
\hfill
\begin{subfigure}[t]{0.32\textwidth}
\begin{tcolorbox}[enhanced, colback=teacherbg, colframe=green!55!black, fonttitle=\bfseries\scriptsize, title={\method~v2 (3{,}576 tok, \textbf{17\%})}, left=3pt, right=3pt, top=2pt, bottom=2pt]
\scriptsize
{\ttfamily\color{gray!60}<think>}

Tangency of the circle to the perpendicular bisector gives $k=\tfrac{23}{8},\tfrac{123}{8}$, summing to $\tfrac{73}{4}$, so $m=73,n=4$.

{\ttfamily\color{gray!60}</think>}

\vspace{1pt}\noindent\textcolor{gray!30}{\rule{\linewidth}{0.3pt}}\vspace{2pt}

$m+n = 73+4 = 77$.

Answer: \textcolor{green!50!black}{$\boldsymbol{77}$} \rlap{\;\textcolor{green!50!black}{\checkmark}}
\end{tcolorbox}
\end{subfigure}
\hfill
\begin{subfigure}[t]{0.32\textwidth}
\begin{tcolorbox}[enhanced, colback=orange!3, colframe=orange!60!black, fonttitle=\bfseries\scriptsize, title={\method~v1 (2{,}507 tok, \textbf{42\%})}, left=3pt, right=3pt, top=2pt, bottom=2pt]
\scriptsize
{\ttfamily\color{gray!60}<think>}

The tangency condition yields $k=\tfrac{23}{8}$ and $k=\tfrac{123}{8}$; their sum is $\tfrac{73}{4}$, so $m=73,n=4$.

{\ttfamily\color{gray!60}</think>}

\vspace{1pt}\noindent\textcolor{gray!30}{\rule{\linewidth}{0.3pt}}\vspace{2pt}

$m+n = 77$.

Answer: \textcolor{green!50!black}{$\boldsymbol{77}$} \rlap{\;\textcolor{green!50!black}{\checkmark}}
\end{tcolorbox}
\end{subfigure}

\caption{\textbf{\method compresses reasoning while preserving the correct answer, and honors both conciseness instructions.} On a MATH-500 problem (top) and an AIME 2024 problem (middle), the base model, the difficulty-aware teacher (v2), and the uniform teacher (v1) all reach the correct answer, but the \method responses are far shorter; v1 compresses most while v2 keeps slightly more reasoning (e.g.\ explicitly noting the triangles do not overlap, or spelling out the run-partition argument), consistent with its instruction not to over-compress. \textbf{Bottom:} an AIME 2025 format-change example, where the base model reports its result only inside ``$\backslash$boxed\{$\cdot$\}'' with no ``Answer:'' line (undercounted by a strict answer-format grader), while both \method variants emit the requested ``Answer:'' line (Appendix~\ref{app:answer_format}).}
\label{fig:wrong_to_correct}
\end{figure}

In addition to math reasoning tasks, we also show our method is very effective in agentic task like Deep Planning \citep{zhang2026deepplanning} in Appendix \ref{app:deepplanning}. 

\subsection{Ablation Study}
\label{sec:ablation}

\subsubsection{How Much Compression Is Too Much? A Sweet Spot}
\label{sec:ablation_sweetspot}

Compression increases steadily over training, but accuracy does not stay flat forever. To locate the useful operating point, we track a single Qwen3-8B run with the difficulty-aware teacher (v2, $M{=}50$) across a full epoch (Figure~\ref{fig:sweetspot}).

\begin{figure}[t]
\centering
\includegraphics[width=\linewidth]{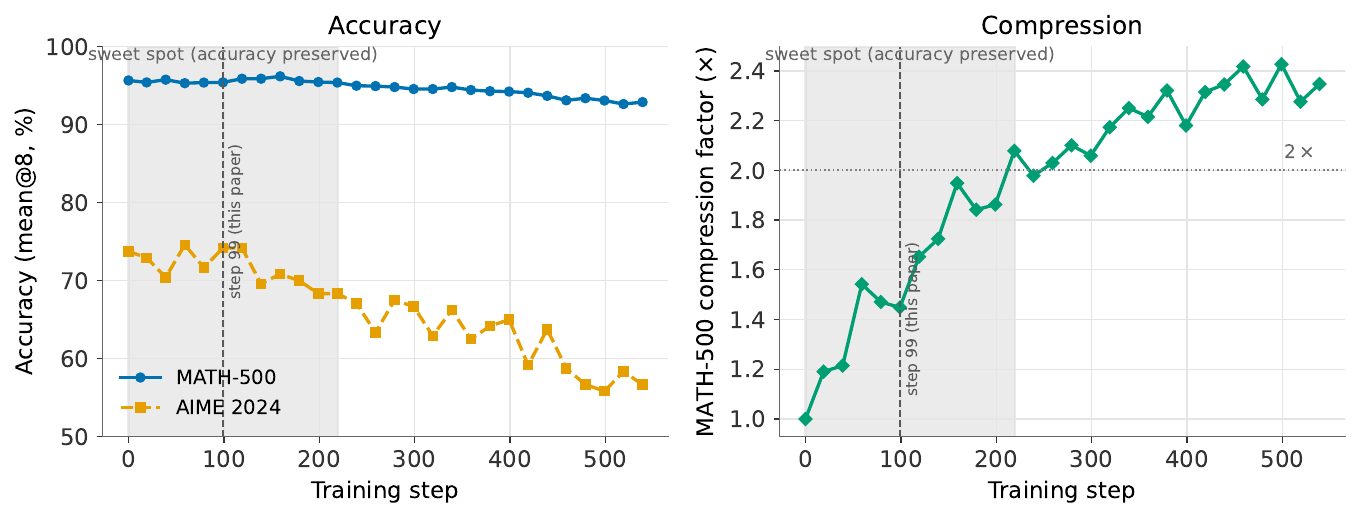}
\caption{\textbf{A compression sweet spot (Qwen3-8B, v2 teacher, full epoch).} \emph{Left:} accuracy over training. \emph{Right:} MATH-500 compression factor (base length / current length). Compression rises steadily with training; accuracy is preserved up to about $2\times$ compression, reached near step 220 (shaded band), and then declines, first on the harder AIME 2024 and later on MATH-500. The step-99 checkpoint used throughout the paper (dashed line) sits well inside this region at ${\sim}1.45\times$ compression. Pushing compression beyond the sweet spot trades accuracy for diminishing extra reduction.}
\label{fig:sweetspot}
\end{figure}

\begin{finding}
\textbf{Finding: compressing about $2\times$ is the sweet spot.} Up to roughly $2\times$ reduction in reasoning length, accuracy is preserved (MATH-500 stays near 95\%, AIME 2024 near its base level). Further training keeps shrinking responses, but the extra compression is modest and accuracy begins to drop, on AIME 2024 first and eventually on MATH-500.
\end{finding}

As training proceeds, MATH-500 length falls from the base $4{,}823$ tokens toward ${\sim}2{,}000$ (a ${\sim}2.4\times$ reduction by the end of the epoch), while MATH-500 accuracy holds near 95\% until about $2\times$ compression and then declines to ${\sim}93\%$. AIME 2024 is more sensitive: its accuracy is preserved (${\sim}70$--$74\%$) through the same region but erodes to ${\sim}57\%$ as compression is pushed to the end of the epoch, with only a few additional points of length reduction gained. The gain past $2\times$ is therefore small and comes at a real accuracy cost, so we stop within the sweet spot: the step-99 checkpoint used throughout this paper sits well inside this region at ${\sim}1.45\times$ compression, on the conservative side of the $2\times$ boundary, delivering strong compression with fully preserved accuracy. This also motivates the periodic-refresh and instruction ablations below, which control \emph{how fast} a run moves along this curve.

\subsubsection{Which Conciseness Instruction Compresses Best?}
\label{sec:ablation_prompt}

The teacher's behavior is shaped entirely by its conciseness instruction $c$, making the wording of $c$ a central design choice. We compare five instructions (Table~\ref{tab:prompt_variants}), each expressing ``be concise'' through a different mechanism: \textbf{v1} gives a single uniform ``be direct, avoid elaboration'' directive; our default \textbf{v2} adds a difficulty-aware caveat to protect hard problems; \textbf{v3} names the text categories that are safe to cut versus must-keep; \textbf{v4} frames the audience as an expert; and \textbf{v5} imposes a terse numbered-list format. All variants are trained identically (Qwen3-8B, full-vocab reverse-KL, $M{=}50$) and evaluated at step 99.

\begin{table}[t]
\centering
\caption{\textbf{Conciseness-instruction variants.} Each teacher prompt $c$ expresses ``be concise'' through a different mechanism. The answer-format suffix (``The last line\ldots Answer:'') is identical across variants and omitted here.}
\label{tab:prompt_variants}
\small
\renewcommand{\arraystretch}{1.25}
\begin{tcolorbox}[enhanced, colback=teacherbg, colframe=green!55!black, boxrule=1pt, left=2pt, right=2pt, top=2pt, bottom=2pt]
\begin{tabularx}{\linewidth}{@{}>{\bfseries}l >{\raggedright\arraybackslash}p{0.20\linewidth} >{\raggedright\arraybackslash}X@{}}
\toprule
\normalfont\textbf{ID} & \textbf{Mechanism} & \textbf{Instruction $c$ (abridged)} \\
\midrule
v1 & Uniform & Solve concisely and correctly. Be direct---avoid unnecessary elaboration, redundant steps, or restating the problem; focus only on the key reasoning steps needed to reach the answer. \\
\addlinespace
v2 & Difficulty-aware (default) & Solve concisely and correctly; prioritize a correct answer over brevity. Be direct on straightforward problems; for hard/multi-step ones do not over-compress---keep case analysis, edge cases, and a final check. \\
\addlinespace
v3 & Redundancy-targeted & Write only what is needed: omit restatements, narration, and obvious arithmetic; keep every logically necessary step (case splits, key idea, a one-line check). \\
\addlinespace
v4 & Expert-reader & Write for an expert who wants only the essential reasoning: state the key steps and decisive computation, skip standard-fact explanations and narration; include required case analysis/checks. \\
\addlinespace
v5 & Terse-structured & Present the solution as a minimal numbered list---each step one short line, no prose, no restatement; keep required cases and end with a one-line verification. \\
\bottomrule
\end{tabularx}
\end{tcolorbox}
\end{table}

\begin{figure}[t]
\centering
\includegraphics[width=\linewidth]{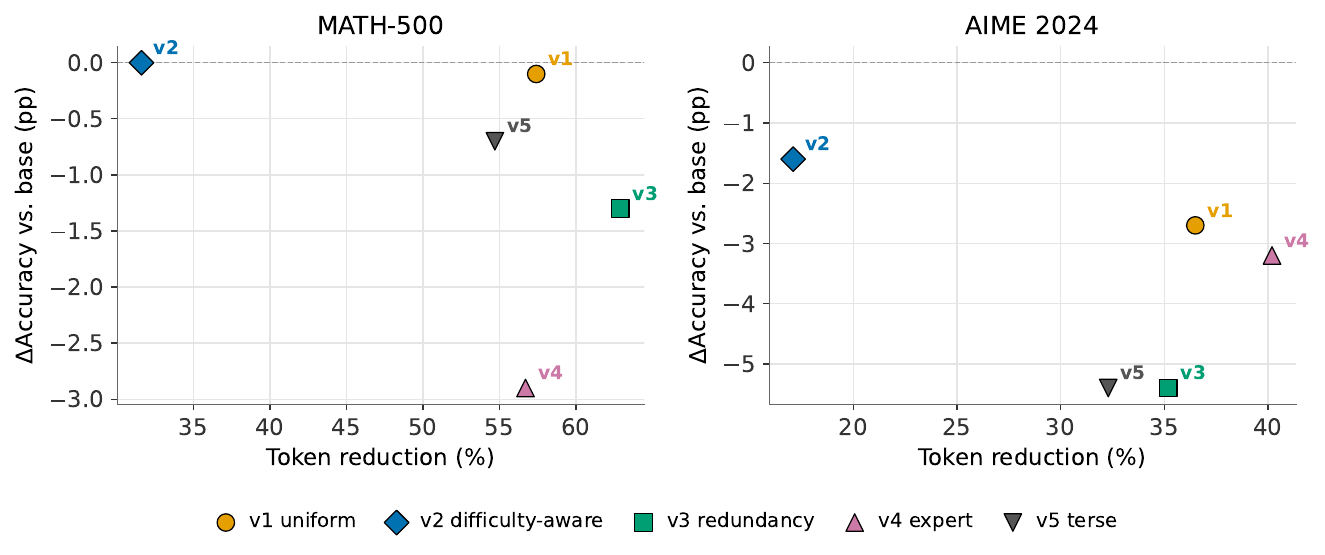}
\caption{\textbf{Conciseness-instruction ablation: accuracy--compression trade-off (Qwen3-8B, step 99, 30K budget).} Each point is one instruction (v1--v5); axes are token reduction (\%, vs.\ base) and accuracy change ($\Delta$Acc, pp vs.\ base mean@8) on \textbf{a}, MATH-500 and \textbf{b}, AIME 2024. Upper-right is better (more compression, less accuracy loss). Base model matches Table~\ref{tab:main} (95.7\% / 4{,}884 tok MATH-500; 76.2\% / 14{,}229 tok AIME 2024).}
\label{fig:prompt_ablation}
\end{figure}

Figure~\ref{fig:prompt_ablation} reveals a clear \textbf{accuracy--compression frontier across instructions}. The difficulty-aware \textbf{v2} default best preserves accuracy ($+0.0\%$ on MATH-500, $-1.6\%$ on AIME 2024) but compresses least (31.6\% / 17.1\%), because its explicit ``do not over-compress on hard problems'' caveat protects deliberation. The uniform \textbf{v1} instruction sits further along the frontier: dropping that caveat nearly doubles MATH-500 compression to 57.4\% (and AIME to 36.5\%) for only a small accuracy cost ($-0.1\%$ / $-2.7\%$), showing that the difficulty-aware guardrail trades substantial compression for a marginal accuracy gain. Pushing harder still, the redundancy-targeted \textbf{v3} reaches the highest MATH-500 compression (62.9\%) but at a real accuracy cost ($-1.3\%$ MATH-500, $-5.4\%$ AIME), as do \textbf{v4} (expert-reader; strongest AIME compression at 40.2\%) and \textbf{v5} (terse format). Overall, no single instruction dominates: v2 is the accuracy-safe default, v1 is the best choice when aggressive compression matters, and the more prescriptive v3--v5 buy extra compression only by giving up accuracy.

\subsubsection{How Sensitive Is Compression to the Teacher Update Interval?}
\label{sec:ablation_M}

The teacher update interval $M$ (Eq.~\ref{eq:teacher_update}) controls how frequently the teacher weights are synchronized with the student. A larger $M$ provides a more stable distillation target but limits progressive compression; a smaller $M$ pushes compression further but risks instability when the teacher chases a rapidly-moving student. We sweep $M \in \{1, 10, 30, 50, 100\}$ on Qwen3-8B (full-vocab reverse-KL, all else fixed). Because we evaluate at step 99, $M{=}100$ never reaches its first refresh and is therefore effectively a \emph{frozen} teacher. Figure~\ref{fig:m_ablation} reports accuracy and token reduction at step 99.

\begin{figure}[t]
\centering
\includegraphics[width=\linewidth]{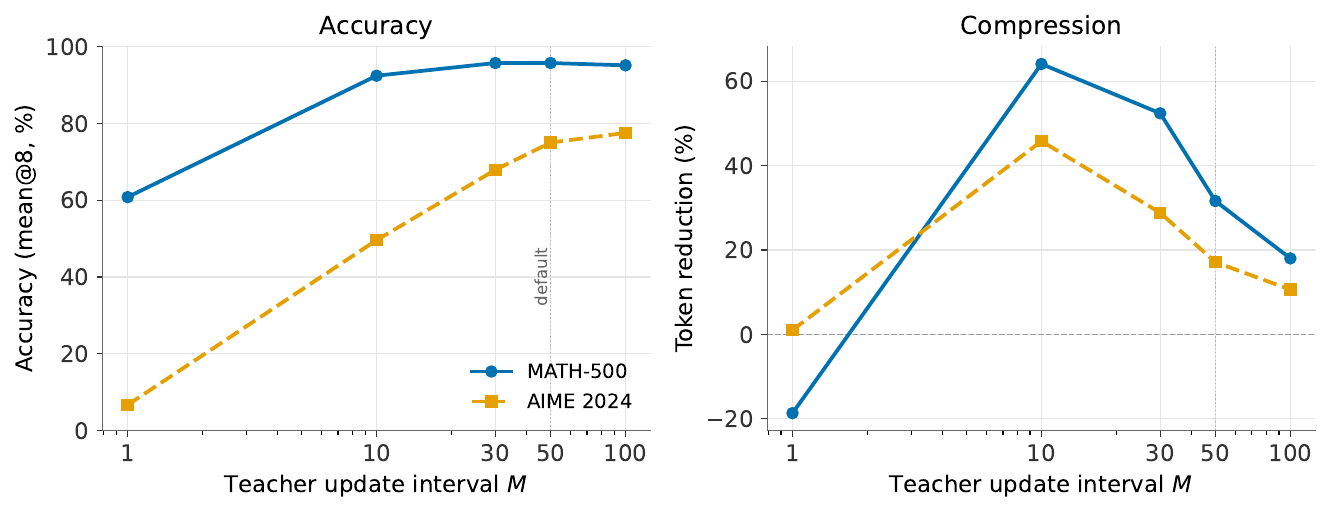}
\caption{\textbf{Teacher update interval $M$ ablation (Qwen3-8B, step 99, 30K budget).} Accuracy (\emph{left}) and token reduction (\emph{right}) vs.\ $M$ on MATH-500 and AIME 2024, with the base taken from Table~\ref{tab:main}. Smaller $M$ compresses more aggressively but, past a point, destabilizes training; $M{=}1$ collapses (MATH-500 response length \emph{grows}, hence negative reduction). At a step-99 evaluation $M{=}100$ has not yet refreshed and is effectively a frozen teacher. We use $M{=}50$ as the default (dotted line).}
\label{fig:m_ablation}
\end{figure}

Figure~\ref{fig:m_ablation} reveals a clear \textbf{compression--stability trade-off governed by $M$}:

\textbf{$M{=}1$ collapses.} Refreshing the teacher after every gradient step creates a moving-target feedback loop in which the student chases a teacher that is itself changing in response to the student. Training then degenerates: MATH-500 accuracy falls to 60.8\% (AIME 2024 to 6.7\%), MATH-500 response length grows rather than shrinks (negative reduction), and AIME compression collapses to near zero. This mirrors the instability reported by \citet{shenfeld2025sdft}.

\textbf{Small but non-trivial $M$ ($10$--$30$) compresses hardest.} $M{=}10$ removes 64\% of MATH-500 tokens and 46\% on AIME 2024, the most aggressive compression in the sweep, but at a real accuracy cost (AIME 2024 drops to 49.6\%). $M{=}30$ recovers most of that accuracy (67.9\%) while still compressing 52\% and 29\%.

\textbf{The frozen teacher ($M{=}100$) preserves accuracy but compresses least.} Since $M{=}100$ never refreshes before step 99, it acts as a single static concise teacher: accuracy stays high (95.1\% MATH-500, 77.5\% AIME 2024) but compression is only 18\% and 11\%. That every refreshing schedule ($M \leq 50$) compresses more confirms that periodic refresh, rather than a one-shot concise teacher, is what drives progressive compression.

We adopt $M{=}50$ as the default: it sits at the knee of this trade-off, retaining base-level accuracy on MATH-500 (95.7\%) and AIME 2024 (75.0\%) while still achieving substantial compression (31.6\% / 17.1\%).

\subsubsection{How Sensitive Is Compression to the Rollout Temperature?}
\label{sec:ablation_temp}

\method generates student rollouts with sampling temperature $\tau$, which controls the diversity of the trajectories the teacher then scores. We sweep $\tau \in \{0.3, 0.6, 1.0, 1.5\}$ with all else fixed (Qwen3-8B, full-vocab reverse-KL, $M{=}50$, v2 instruction), evaluating at step 99 (Figure~\ref{fig:temp_ablation}).

\begin{figure}[t]
\centering
\includegraphics[width=\linewidth]{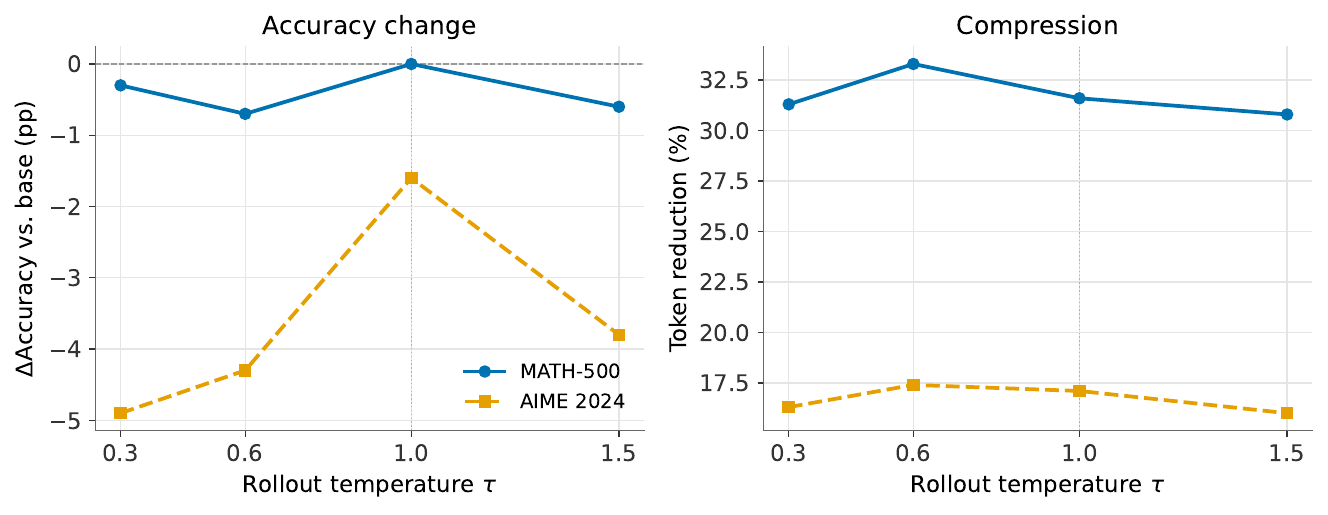}
\caption{\textbf{Rollout temperature ablation (Qwen3-8B, step 99, 30K budget).} Accuracy change ($\Delta$Acc vs.\ base mean@8, \emph{left}) and token reduction (\emph{right}) vs.\ the rollout temperature $\tau$ on MATH-500 and AIME 2024. Base model identical to Table~\ref{tab:main}. The curves are nearly flat, showing robustness across $\tau$; the default $\tau{=}1.0$ (dotted line) is the best operating point.}
\label{fig:temp_ablation}
\end{figure}

\method is robust to rollout temperature. Across the full $\tau \in [0.3, 1.5]$ range, MATH-500 accuracy stays within $0.7$\,pp of the base model and compression varies by less than $3$\,pp ($30.8$ to $33.3\%$). The default $\tau{=}1.0$ is the best overall operating point: it gives the smallest AIME accuracy drop ($-1.6\%$) at compression comparable to the rest of the sweep ($17.1\%$). Moderate sampling diversity exposes the teacher to a representative spread of student trajectories, without the degenerate sampling that very high or very low temperatures induce. We therefore keep $\tau{=}1.0$ as the default.

\subsubsection{How Sensitive Is Compression to the Rollout Length?}
\label{sec:ablation_len}

Because \method optimizes a per-token KL objective rather than an outcome reward, it does not require complete rollouts; partial trajectories already supply a training signal (Section~\ref{sec:setup}). We test how short the training rollouts can be by capping the student response length during training at $\{1\text{k}, 4\text{k}, 8\text{k}, 30\text{k}\}$ tokens, all else fixed (Qwen3-8B, $M{=}50$, v2), and evaluating at step 99 under the common 30K eval budget (Figure~\ref{fig:len_ablation}).

\begin{figure}[t]
\centering
\includegraphics[width=\linewidth]{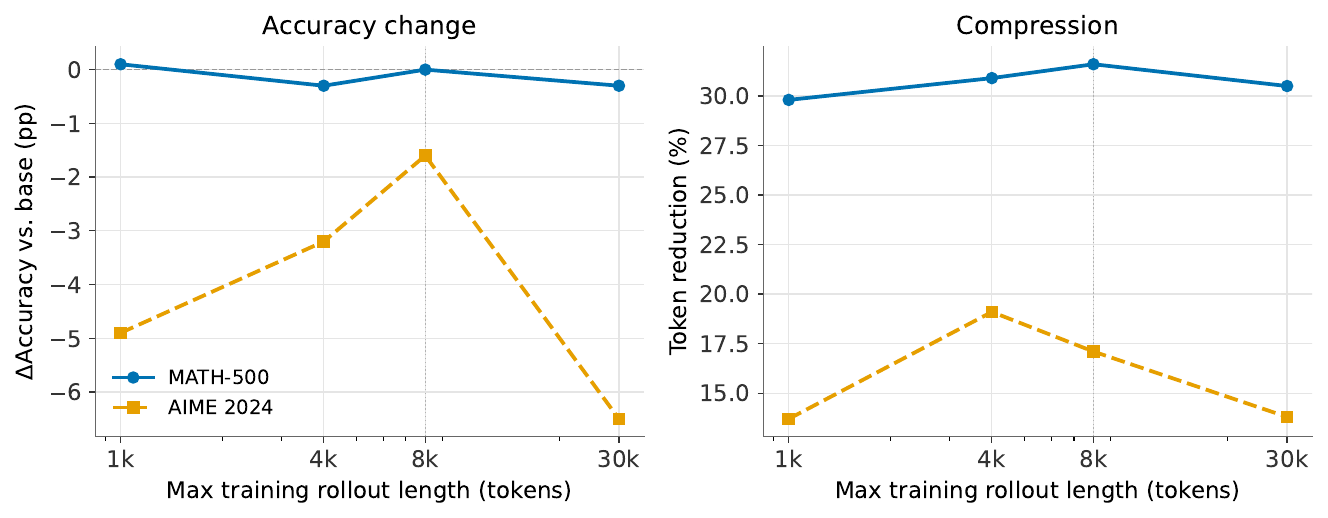}
\caption{\textbf{Training rollout-length ablation (Qwen3-8B, step 99, 30K eval budget).} Accuracy change ($\Delta$Acc vs.\ base mean@8, \emph{left}) and token reduction (\emph{right}) vs.\ the maximum \emph{training} rollout length on MATH-500 and AIME 2024; all variants are evaluated under the same 30K budget. Base model identical to Table~\ref{tab:main}. Even a 1k-token cap trains well; the $8$k default (dotted line) is the best operating point, while the longest 30k cap is not better.}
\label{fig:len_ablation}
\end{figure}

The result confirms that short training rollouts suffice. Even a $1$k-token cap, roughly a fifth of the base model's MATH-500 length, recovers $29.8\%$ MATH-500 compression with no accuracy loss, which validates the partial-rollout design. The $8$k default is the best setting: it gives the strongest MATH-500 compression ($31.6\%$) and the smallest AIME accuracy drop ($-1.6\%$), with AIME compression ($17.1\%$) on par with the $4$k cap. The longest cap ($30$k) is not better; its AIME accuracy drops most ($-6.5\%$) while compressing least ($13.8\%$), because unconstrained-length rollouts let the student reinforce its own verbosity before the teacher signal takes effect. This makes \method cheaper to train than outcome-reward RL, which needs full-length completions to compute a terminal reward.

\section{Limitations and Future Work}
\label{sec:limitations}

We discuss the scope of the current study and natural directions for future work.

\paragraph{Instruction-following as an enabler.} \method depends on the base model's ability to follow the conciseness instruction, since the teacher signal comes entirely from that instruction. Models that follow instructions well produce a clear teacher distribution and compress reliably, as we observe across Qwen3-8B, Qwen3-14B, and DeepSeek-R1-Distill-Llama-8B. Models with weak instruction-following would provide a weaker teacher signal. Characterizing the minimum instruction-following capability needed for effective self-distillation is a direction for future work.

\paragraph{Progressive compression dynamics.} The periodic teacher update (Eq.~\ref{eq:teacher_update}) pushes compression beyond a frozen teacher, and the method is robust across a range of moderate update intervals ($M \in \{30, 50\}$; Section~\ref{sec:ablation_M}). The compression signal on harder benchmarks such as AIME is weaker because the teacher itself requires more extensive reasoning. We view this as a feature of difficulty-adaptive compression (Section~\ref{sec:difficulty}) rather than a limitation.

\paragraph{Scope of evaluation.} This paper focuses on mathematical reasoning as a controlled testbed where accuracy can be verified precisely. The design of \method is domain-agnostic, since it requires only problem prompts and a conciseness instruction, so it applies to other reasoning domains such as code generation and scientific question answering where ground-truth verification is unavailable. The main properties of \method (no ground-truth requirement, difficulty adaptivity, and entropy preservation) are structural and do not depend on the evaluation domain. Extending the empirical evaluation to broader reasoning tasks is a natural next step.

\paragraph{Teacher quality characterization.} Our experiments consistently show that the conciseness-conditioned teacher improves accuracy (Table~\ref{tab:main}), and the theoretical analysis (Theorem~\ref{thm:accuracy}) provides formal bounds on accuracy preservation. A finer-grained characterization of when and why conciseness instructions improve versus degrade accuracy across different model families would further strengthen the understanding of self-distillation dynamics.

\section{Conclusion}
\label{sec:conclusion}

We presented \method, an on-policy self-distillation method that compresses reasoning by distilling a model's own concise behavior back into itself. The method uses a conciseness instruction to define a teacher and minimizes per-token reverse KL to that teacher on the student's own rollouts, without ground-truth answers, token budgets, or difficulty estimators. Across three model families, \method shortens reasoning traces substantially while preserving accuracy, and it improves accuracy when the base model has room to improve.

Two conclusions follow from these results. First, a large fraction of the tokens that reasoning models produce are redundant, and removing them preserves or improves accuracy rather than degrading it. Second, models already contain a concise reasoning mode that a conciseness instruction can elicit, and on-policy self-distillation can make that mode the default without reducing entropy (Appendix~\ref{app:entropy}) or general capability.

Because \method supervises only on a conciseness instruction and the model's own rollouts, it applies to domains where ground-truth answers or reliable verifiers are unavailable, provided the model can follow the instruction.


\bibliographystyle{abbrvnat}

\newpage
\appendix

\section{Theoretical Analysis}
\label{sec:theory_main}

We provide a formal analysis of \method's key properties: the connection between the per-token training loss and sequence-level divergence (Section~\ref{sec:kl_chain}), accuracy preservation guarantees (Section~\ref{sec:accuracy_bound}), formalization of difficulty-adaptive compression (Section~\ref{sec:difficulty_theory}), forgetting bounds relative to the base model (Section~\ref{sec:forgetting_bound}), and a probabilistic model of how compression reduces compounding errors (Section~\ref{sec:compounding_theory}).

\subsection{Sequence-Level Divergence and the Training Objective}
\label{sec:kl_chain}

We first establish that the per-token \method objective (Eq.~\ref{eq:loss}) is equivalent to minimizing the sequence-level KL divergence between student and teacher. This identification enables the application of standard information-theoretic tools to the per-token loss.

\begin{lemma}[Chain rule of KL for autoregressive models]
\label{lem:chain_rule}
For autoregressive distributions $q(y \mid x) = \prod_{t=1}^{|y|} q(y_t \mid x, y_{<t})$ and $p(y \mid x) = \prod_{t=1}^{|y|} p(y_t \mid x, y_{<t})$ over the same token vocabulary, the sequence-level KL divergence decomposes as:
\begin{equation}
D_\KL\big(q(\cdot \mid x) \;\big\|\; p(\cdot \mid x)\big) = \EE_{y \sim q}\left[\sum_{t=1}^{|y|} D_\KL\big(q(\cdot \mid x, y_{<t}) \;\big\|\; p(\cdot \mid x, y_{<t})\big)\right].
\end{equation}
\end{lemma}

\begin{proof}
By definition of KL divergence and the autoregressive factorization:
\begin{align}
D_\KL(q \| p) &= \EE_{y \sim q}\!\left[\log \frac{q(y \mid x)}{p(y \mid x)}\right] = \EE_{y \sim q}\!\left[\sum_{t=1}^{|y|} \log \frac{q(y_t \mid x, y_{<t})}{p(y_t \mid x, y_{<t})}\right] \\
&= \EE_{y \sim q}\!\left[\sum_{t=1}^{|y|} D_\KL\big(q(\cdot \mid x, y_{<t}) \| p(\cdot \mid x, y_{<t})\big)\right],
\end{align}
where the second equality uses $\log \prod_t = \sum_t \log$, and the last step recognizes each summand as a per-token KL divergence evaluated at the sampled prefix $y_{<t}$.
\end{proof}

\begin{corollary}
\label{cor:loss_equals_kl}
Identifying $q = \pi_\theta(\cdot \mid x)$ and $p = \pi_{\bar{\theta}}(\cdot \mid x, c)$, the \method training loss (Eq.~\ref{eq:loss}) equals the expected sequence-level KL divergence:
\begin{equation}
\Lcal(\theta) = \EE_{x \sim \D}\big[D_\KL\big(\pi_\theta(\cdot \mid x) \;\big\|\; \pi_{\bar{\theta}}(\cdot \mid x, c)\big)\big].
\end{equation}
\end{corollary}

\subsection{Accuracy Preservation under Compression}
\label{sec:accuracy_bound}

We show that if self-distillation converges (the training loss is small) and the concise teacher preserves accuracy, then the student's accuracy is guaranteed to remain close to the base model's.

\begin{definition}[Accuracy]
\label{def:accuracy}
For a problem distribution $\D$ with correct-answer sets $\{A(x)\}_{x \in \D}$, the accuracy of policy $\pi$ is $\mathrm{Acc}(\pi) = \EE_{x \sim \D}\big[\pi(A(x) \mid x)\big]$, where $\pi(A(x) \mid x) = \sum_{y \in A(x)} \pi(y \mid x)$.
\end{definition}

\begin{theorem}[Accuracy preservation]
\label{thm:accuracy}
Let $\pi_{\theta^*}$ denote the converged student with training loss $\Lcal(\theta^*) \leq \epsilon_\KL$. Suppose the concise teacher preserves accuracy relative to the base model:
\begin{equation}
\mathrm{Acc}\big(\pi_{\bar{\theta}}(\cdot \mid \cdot, c)\big) \geq \mathrm{Acc}(\pi_{\bar{\theta}}) - \epsilon_T.
\end{equation}
Then the student satisfies:
\begin{equation}
\label{eq:acc_bound}
\mathrm{Acc}(\pi_{\theta^*}) \geq \mathrm{Acc}(\pi_{\bar{\theta}}) - \epsilon_T - \sqrt{\frac{\epsilon_\KL}{2}}.
\end{equation}
\end{theorem}

\begin{proof}
By Corollary~\ref{cor:loss_equals_kl}, we have $\EE_{x \sim \D}\big[D_\KL(\pi_{\theta^*}(\cdot \mid x) \| \pi_{\bar{\theta}}(\cdot \mid x, c))\big] \leq \epsilon_\KL$.

\textit{Step 1: KL to total variation.} For each problem $x$, Pinsker's inequality gives:
\begin{equation}
d_\mathrm{TV}\big(\pi_{\theta^*}(\cdot \mid x),\; \pi_{\bar{\theta}}(\cdot \mid x, c)\big) \leq \sqrt{\tfrac{1}{2}\, D_\KL\big(\pi_{\theta^*}(\cdot \mid x) \| \pi_{\bar{\theta}}(\cdot \mid x, c)\big)}.
\end{equation}
Taking expectations over $x \sim \D$ and applying Jensen's inequality (using concavity of $\sqrt{\cdot}$):
\begin{equation}
\label{eq:tv_bound}
\EE_x\big[d_\mathrm{TV}\big(\pi_{\theta^*}(\cdot \mid x),\; \pi_{\bar{\theta}}(\cdot \mid x, c)\big)\big] \leq \sqrt{\tfrac{1}{2}\,\EE_x\!\big[D_\KL\big(\pi_{\theta^*}(\cdot \mid x) \| \pi_{\bar{\theta}}(\cdot \mid x, c)\big)\big]} \leq \sqrt{\tfrac{\epsilon_\KL}{2}}.
\end{equation}

\textit{Step 2: Total variation to accuracy.} Since total variation bounds the difference in probability of any event, in particular the correctness event $A(x)$:
\begin{equation}
\big|\pi_{\theta^*}(A(x) \mid x) - \pi_{\bar{\theta}}(A(x) \mid x, c)\big| \leq d_\mathrm{TV}\big(\pi_{\theta^*}(\cdot \mid x),\; \pi_{\bar{\theta}}(\cdot \mid x, c)\big).
\end{equation}
Taking expectations over $x \sim \D$ and using $|\EE[f]| \leq \EE[|f|]$:
\begin{equation}
\big|\mathrm{Acc}(\pi_{\theta^*}) - \mathrm{Acc}\big(\pi_{\bar{\theta}}(\cdot \mid \cdot, c)\big)\big| \leq \EE_x\big[d_\mathrm{TV}\big(\pi_{\theta^*}(\cdot \mid x),\; \pi_{\bar{\theta}}(\cdot \mid x, c)\big)\big] \leq \sqrt{\tfrac{\epsilon_\KL}{2}}.
\end{equation}

\textit{Step 3: Combine with teacher quality.} From the assumption $\mathrm{Acc}(\pi_{\bar{\theta}}(\cdot \mid \cdot, c)) \geq \mathrm{Acc}(\pi_{\bar{\theta}}) - \epsilon_T$:
\begin{align}
\mathrm{Acc}(\pi_{\theta^*}) &\geq \mathrm{Acc}\big(\pi_{\bar{\theta}}(\cdot \mid \cdot, c)\big) - \sqrt{\tfrac{\epsilon_\KL}{2}} \geq \mathrm{Acc}(\pi_{\bar{\theta}}) - \epsilon_T - \sqrt{\tfrac{\epsilon_\KL}{2}}. \qedhere
\end{align}
\end{proof}

\begin{remark}[When the bound is vacuous, and why that is informative]
\label{rem:vacuous}
Theorem~\ref{thm:accuracy} identifies two independent sources of potential accuracy loss: the teacher accuracy gap $\epsilon_T$ and the student--teacher divergence $\sqrt{\epsilon_\KL/2}$. In our experiments, $\epsilon_T$ is consistently \emph{negative} (the concise teacher is more accurate than the base model), making the bound vacuous in the traditional sense. This is informative rather than a weakness: it reveals why \method improves accuracy. The bound becomes $\mathrm{Acc}(\pi_{\theta^*}) \geq \mathrm{Acc}(\pi_{\bar{\theta}}) + |\epsilon_T| - \sqrt{\epsilon_\KL/2}$, so accuracy improves whenever the teacher's accuracy gain exceeds the distillation gap. The theorem is most useful in the regime of aggressive compression where $\epsilon_T$ might turn positive; it then quantifies the worst-case accuracy degradation.
\end{remark}

\subsection{Difficulty-Adaptive Compression}
\label{sec:difficulty_theory}

We formalize the empirical observation (Table~\ref{tab:main}) that \method compresses easy problems aggressively while preserving reasoning on hard problems.

\begin{definition}[Essential and compressible tokens]
\label{def:essential}
For problem $x$ and student rollout $y \sim \pi_\theta(\cdot \mid x)$, classify each token position $t$ based on the implicit reward sign (Theorem~\ref{thm:reward}):
\begin{align}
\mathcal{E}(x, y) &= \big\{t : \pi_{\bar{\theta}}(y_t \mid x, c, y_{<t}) \geq \pi_\theta(y_t \mid x, y_{<t})\big\} & &\text{(essential: } r(y_t, x) \geq 0\text{)}, \\
\mathcal{C}(x, y) &= \big\{t : \pi_{\bar{\theta}}(y_t \mid x, c, y_{<t}) < \pi_\theta(y_t \mid x, y_{<t})\big\} & &\text{(compressible: } r(y_t, x) < 0\text{)}.
\end{align}
\end{definition}

\begin{proposition}[Difficulty-adaptive compression signal]
\label{prop:difficulty}
Let $d(x) \in [0, 1]$ denote problem difficulty, defined as the base model's failure rate $d(x) = 1 - \pi_{\theta_0}(A(x) \mid x)$. Assume:
\begin{itemize}[leftmargin=*, itemsep=2pt]
    \item[\textup{(A1)}] \textbf{Essential fraction increases with difficulty:} The expected fraction of essential tokens $\rho(x) \coloneqq \EE_{y}[|\mathcal{E}(x, y)| / |y|]$ is non-decreasing in $d(x)$.
    \item[\textup{(A2)}] \textbf{Category-level KL is problem-independent:} There exist constants $D_\mathcal{E}, D_\mathcal{C} > 0$ such that for all problems $x$:
    \begin{align}
    \EE\big[D_\KL\big(\pi_\theta(\cdot \mid x, y_{<t}) \| \pi_{\bar{\theta}}(\cdot \mid x, c, y_{<t})\big) \,\big|\, t \in \mathcal{C}(x, y)\big] &= D_\mathcal{C}, \\
    \EE\big[D_\KL\big(\pi_\theta(\cdot \mid x, y_{<t}) \| \pi_{\bar{\theta}}(\cdot \mid x, c, y_{<t})\big) \,\big|\, t \in \mathcal{E}(x, y)\big] &= D_\mathcal{E}.
    \end{align}
    \item[\textup{(A3)}] \textbf{Compressible tokens carry strictly larger KL:} $D_\mathcal{C} > D_\mathcal{E}$.
\end{itemize}
Then the expected normalized compression signal
\begin{equation}
S(x) = \EE_{y \sim \pi_\theta(\cdot \mid x)}\!\left[\frac{1}{|y|}\sum_{t=1}^{|y|} D_\KL\big(\pi_\theta(\cdot \mid x, y_{<t}) \| \pi_{\bar{\theta}}(\cdot \mid x, c, y_{<t})\big)\right]
\end{equation}
is non-increasing in $d(x)$.
\end{proposition}

\begin{proof}
Decompose the normalized KL into essential and compressible contributions. For a given rollout $y$:
\begin{equation}
\frac{1}{|y|}\sum_t D_\KL(q_t \| p_t) = \underbrace{\frac{|\mathcal{E}|}{|y|} \cdot \bar{D}_\mathcal{E}}_{\text{essential term}} + \underbrace{\frac{|\mathcal{C}|}{|y|} \cdot \bar{D}_\mathcal{C}}_{\text{compressible term}},
\end{equation}
where $q_t = \pi_\theta(\cdot \mid x, y_{<t})$, $p_t = \pi_{\bar{\theta}}(\cdot \mid x, c, y_{<t})$, and $\bar{D}_\mathcal{E}, \bar{D}_\mathcal{C}$ denote the average per-token KL on essential and compressible tokens in this rollout, respectively. Taking expectations over $y$ and applying assumption~(A2):
\begin{align}
S(x) &= \rho(x) \cdot D_\mathcal{E} + (1-\rho(x)) \cdot D_\mathcal{C} \\
&= D_\mathcal{C} - \rho(x) \cdot (D_\mathcal{C} - D_\mathcal{E}).
\end{align}
By~(A3), $D_\mathcal{C} - D_\mathcal{E} > 0$, so $S(x)$ is a strictly decreasing affine function of $\rho(x)$. Since $\rho(x)$ is non-decreasing in $d(x)$ by~(A1), $S(x)$ is non-increasing in $d(x)$.

\emph{Quantitatively}, for two problems with difficulties $d_1 < d_2$ (hence $\rho(x_1) \leq \rho(x_2)$ by A1):
\begin{equation}
S(x_1) - S(x_2) = \big(\rho(x_2) - \rho(x_1)\big) \cdot \big(D_\mathcal{C} - D_\mathcal{E}\big) \geq 0. \qedhere
\end{equation}
\end{proof}

\begin{remark}
Assumption~(A1), that harder problems have a larger fraction of essential reasoning steps, is the core structural assumption. It is empirically supported by the monotonic compression--difficulty relationship in Table~\ref{tab:main}: on Qwen3-14B, MATH-500 (up to 56\% compression) is compressed more than the harder AIME benchmarks (32--38\%). Assumption~(A2) posits that the average KL divergence on essential and compressible tokens depends on token \emph{category} (essential vs.\ compressible) rather than on the specific problem. This is a modeling simplification; in practice, per-token KL values vary across problems, and the proposition should be interpreted as holding for the category-averaged quantities. Assumption~(A3) states that compressible tokens exhibit a larger vocabulary-level KL divergence than essential tokens. Note that this does not follow directly from the per-token reward sign: the essential/compressible classification (Definition~\ref{def:essential}) is based on the probability of the \emph{sampled} token $y_t$, whereas the KL divergence $D_\KL(q_t \| p_t)$ is an expectation over the \emph{entire vocabulary} at position $t$. Rather, (A3) is a structural modeling assumption, motivated by the intuition that at positions where the teacher concentrates mass differently from the student (compressible positions), the full distributional divergence tends to be larger than at positions where both distributions agree on the dominant token (essential positions).
\end{remark}

\subsection{Bounded Forgetting from the Base Model}
\label{sec:forgetting_bound}

A central advantage of on-policy self-distillation over off-policy SFT is controlled divergence from the original model. We formalize this through the \emph{conciseness gap}.

\begin{definition}[Conciseness gap]
\label{def:conciseness_gap}
The conciseness gap of the base model $\pi_{\theta_0}$ under instruction $c$ on input $x$ is
\begin{equation}
\gamma(x) = d_\mathrm{TV}\big(\pi_{\theta_0}(\cdot \mid x),\; \pi_{\theta_0}(\cdot \mid x, c)\big).
\end{equation}
\end{definition}

\begin{proposition}[Bounded forgetting under on-policy self-distillation]
\label{prop:forgetting}
Consider the first teacher window where $\bar{\theta} = \theta_0$ (frozen teacher). If the converged \method loss satisfies $\Lcal(\theta^*) \leq \epsilon_\KL$, then:
\begin{equation}
\label{eq:forgetting}
\EE_{x \sim \D}\big[d_\mathrm{TV}\big(\pi_{\theta^*}(\cdot \mid x),\; \pi_{\theta_0}(\cdot \mid x)\big)\big] \leq \sqrt{\frac{\epsilon_\KL}{2}} + \EE_{x \sim \D}[\gamma(x)].
\end{equation}
For subsequent windows with periodic teacher update, the same bound holds with $\theta_0$ replaced by the teacher weights $\bar{\theta}$ at the start of that window. Moreover, $\gamma(x)$ is difficulty-adaptive: for hard problems where the conciseness instruction has little effect, $\gamma(x) \approx 0$, so forgetting is minimal.
\end{proposition}

\begin{proof}
By the triangle inequality for total variation distance:
\begin{equation}
d_\mathrm{TV}\big(\pi_{\theta^*}(\cdot \mid x),\; \pi_{\theta_0}(\cdot \mid x)\big) \leq \underbrace{d_\mathrm{TV}\big(\pi_{\theta^*}(\cdot \mid x),\; \pi_{\theta_0}(\cdot \mid x, c)\big)}_{\text{student--teacher gap}} + \underbrace{d_\mathrm{TV}\big(\pi_{\theta_0}(\cdot \mid x, c),\; \pi_{\theta_0}(\cdot \mid x)\big)}_{\gamma(x)}.
\end{equation}
The student--teacher gap is bounded via Pinsker's inequality. By Corollary~\ref{cor:loss_equals_kl}, $\EE_x[D_\KL(\pi_{\theta^*}(\cdot \mid x) \| \pi_{\theta_0}(\cdot \mid x, c))] \leq \epsilon_\KL$. Applying Pinsker to each $x$ and Jensen's inequality over $x \sim \D$:
\begin{equation}
\EE_x\big[d_\mathrm{TV}\big(\pi_{\theta^*}(\cdot \mid x),\; \pi_{\theta_0}(\cdot \mid x, c)\big)\big] \leq \sqrt{\tfrac{\epsilon_\KL}{2}}.
\end{equation}
Taking expectations on both sides of the triangle inequality yields~\eqref{eq:forgetting}.

For the difficulty-adaptive claim: on hard problems, the conciseness instruction cannot substantially alter the output distribution because most reasoning steps are essential. Thus $\pi_{\theta_0}(\cdot \mid x, c) \approx \pi_{\theta_0}(\cdot \mid x)$, giving $\gamma(x) \approx 0$.
\end{proof}

\begin{remark}[Comparison with off-policy SFT]
\label{rem:forgetting_comparison}
Standard off-policy SFT minimizes $-\EE_{(x,y) \sim \D_T}[\log \pi_\theta(y \mid x)]$ on a fixed teacher dataset $\D_T$. The analogous forgetting decomposition is:
\begin{equation}
d_\mathrm{TV}\big(\pi_{\theta_\mathrm{SFT}}(\cdot \mid x),\; \pi_{\theta_0}(\cdot \mid x)\big) \leq d_\mathrm{TV}\big(\pi_{\theta_\mathrm{SFT}}(\cdot \mid x),\; \D_T(\cdot \mid x)\big) + d_\mathrm{TV}\big(\D_T(\cdot \mid x),\; \pi_{\theta_0}(\cdot \mid x)\big).
\end{equation}
The second term measures the \emph{distribution mismatch} between the teacher's data and the base model's outputs, which can be substantially larger than the conciseness gap $\gamma(x)$, particularly when the teacher generates qualitatively different reasoning styles. Crucially, off-policy SFT drives $\pi_{\theta_\mathrm{SFT}}$ toward $\D_T$ directly, while on-policy distillation generates data from the student's own evolving distribution, inherently limiting divergence from the base model~\citep{shenfeld2025sdft}.
\end{remark}

\subsection{Compression Reduces Compounding Error}
\label{sec:compounding_theory}
We provide a simple probabilistic model that explains the most striking empirical finding: shorter reasoning traces can \emph{improve} accuracy rather than degrade it.

\begin{proposition}[Shorter traces reduce error accumulation]
\label{prop:compounding}
Consider an autoregressive reasoning model where each token position independently introduces a reasoning error (an incorrect intermediate step that corrupts the final answer) with probability $p_\mathrm{err} \in (0, 1)$. For a trace of length $L$, the probability of producing a correct answer is:
\begin{equation}
\mathrm{Acc}(L) = (1 - p_\mathrm{err})^L.
\end{equation}
If compression reduces the trace from $L$ to $\alpha L$ tokens ($\alpha \in (0,1)$) without increasing the per-token error rate, the accuracy ratio satisfies:
\begin{equation}
\label{eq:compounding}
\frac{\mathrm{Acc}(\alpha L)}{\mathrm{Acc}(L)} = (1 - p_\mathrm{err})^{-(1-\alpha)L} \geq 1 + (1-\alpha)L \cdot p_\mathrm{err}.
\end{equation}
The accuracy improvement grows \emph{exponentially} in the number of removed tokens $(1-\alpha)L$.
\end{proposition}

\begin{proof}
Direct computation gives the exact ratio:
\begin{equation}
\frac{\mathrm{Acc}(\alpha L)}{\mathrm{Acc}(L)} = \frac{(1-p_\mathrm{err})^{\alpha L}}{(1-p_\mathrm{err})^L} = (1-p_\mathrm{err})^{-(1-\alpha)L}.
\end{equation}
Let $m = (1-\alpha)L > 0$. Since $\ln(1-p) \leq -p$ for all $p \in (0,1)$ (which follows from the concavity of $\ln$ and the tangent line at $p=0$), we have $-\ln(1-p_\mathrm{err}) \geq p_\mathrm{err}$, giving:
\begin{equation}
(1-p_\mathrm{err})^{-m} = e^{-m \ln(1-p_\mathrm{err})} \geq e^{m \cdot p_\mathrm{err}} \geq 1 + m \cdot p_\mathrm{err},
\end{equation}
where the final inequality is $e^u \geq 1 + u$ for all $u \geq 0$.
\end{proof}

\begin{remark}[Quantitative interpretation]
On MATH-500 (Qwen3-14B), the base model generates $L \approx 4{,}139$ tokens and \method compresses to $\alpha \approx 0.44$, removing $(1-\alpha)L \approx 2{,}330$ tokens. Even with a modest per-token error rate of $p_\mathrm{err} = 10^{-4}$, the linear lower bound gives an accuracy ratio $\geq 1.23$, a 23\% relative improvement. The exponential term $(1-p_\mathrm{err})^{-2330} \approx e^{0.23} \approx 1.26$ provides the tighter bound. On AIME benchmarks where $L \approx 13{,}000$ tokens, the same per-token error rate yields even larger potential gains per token removed.
\end{remark}

\begin{remark}[Conservative nature of the independence assumption]
The independence assumption is conservative: in practice, reasoning errors compound because one incorrect intermediate step causes subsequent steps to build on a false premise, amplifying the error. Under positively correlated errors, the benefit of removing error-prone tokens exceeds the independent-error prediction, making Proposition~\ref{prop:compounding} a \emph{lower bound} on the improvement from compression. This is consistent with the empirical accuracy gains where the base model has room to improve (e.g., $71.3{\to}82.1$ on MATH-500 for DeepSeek-R1-Distill-Llama-8B; Table~\ref{tab:main}), which exceed what the simple independent-error model predicts.
\end{remark}

\section{Survey of Reasoning Compression Methods}
\label{app:survey}

Table~\ref{tab:survey} summarizes 13 reasoning compression methods along four axes. This survey motivates the design of \method by revealing the pervasive dependence on ground-truth answers and the rarity of difficulty-adaptive methods.

\begin{table}[h]
\centering
\caption{Survey of 13 reasoning compression methods. LP = Length Penalty in reward; DD = Difficulty-Dependent; CA = Correct Answer required; HB = Hard Budget.}
\label{tab:survey}
\small
\begin{tabular}{@{}l cccc l@{}}
\toprule
\textbf{Method} & \textbf{LP} & \textbf{DD} & \textbf{CA} & \textbf{HB} & \textbf{Approach} \\
\midrule
L1 & \checkmark & & \checkmark & \checkmark & RL \\
DiPO & \checkmark & \checkmark & \checkmark & & RL \\
DIET & \checkmark & \checkmark & \checkmark & & RL \\
DLER & \checkmark & \checkmark & \checkmark & & RL \\
Leash & \checkmark & & \checkmark & & RL \\
\midrule
SEER & & & \checkmark & & SFT \\
TokenSkip & & & \checkmark & \checkmark & SFT \\
S3-CoT & & & & & Steering \\
DAP/LiteCoT & & \checkmark & \checkmark & & SFT \\
Chain of Draft & & & & & Prompt \\
TrimR & & & & & Inference \\
NoWait & & & & & Inference \\
FlowSteer & & & & & Inference \\
\midrule
\textbf{\method (Ours)} & & \checkmark & & & Self-distill \\
\bottomrule
\end{tabular}
\end{table}

\paragraph{Comparability.} We do not tabulate the reported accuracy/token-reduction numbers of these methods side by side, because they are not directly comparable: with the sole exception of NoWait, no cited method uses any of our base models (Qwen3-8B/14B, DeepSeek-R1-Distill-Llama-8B), and papers differ in benchmark splits (full MATH vs.\ MATH-500, composite AIME sets, MMLU-Pro vs.\ MMLU, non-Diamond GPQA) and evaluation protocols. NoWait is the only baseline evaluated on a \method base model (Qwen3-8B); as reported in its paper, it reduces tokens but loses accuracy on the hard sets (AIME 2025 $74.6{\to}60.0$), whereas \method compresses while preserving accuracy.



\section{Full Main Results}
\label{app:full_main}

Table~\ref{tab:main_full} gives the complete version of Table~\ref{tab:main}: it adds the average
response length (Len) for every row and the inference-only ``concise prompt'' baselines, where the
conciseness instruction is prepended at inference with no training. The concise-prompt rows show that
prompting alone already shortens responses, and that \method extends this compression further while
better preserving accuracy.

\begin{table*}[t]
\centering
\caption{\textbf{Full main results (token budget = 30K).} Accuracy (Acc, mean@8, \%), average reasoning token length (Len), and token reduction relative to the base model (Red., \%; ``---'' marks the base reference). ``Concise prompt'' uses the conciseness instruction at inference only (no training); \method trains with periodic teacher update ($M{=}50$). v1 and v2 denote the uniform and difficulty-aware conciseness instructions.}
\label{tab:main_full}
\small
\setlength{\tabcolsep}{2.5pt}
\resizebox{\linewidth}{!}{%
\begin{tabular}{@{}l ccc ccc ccc ccc ccc@{}}
\toprule
& \multicolumn{3}{c}{\textbf{MATH-500}} & \multicolumn{3}{c}{\textbf{AIME 2024}} & \multicolumn{3}{c}{\textbf{AIME 2025}} & \multicolumn{3}{c}{\textbf{GPQA-D}} & \multicolumn{3}{c}{\textbf{MMLU}} \\
\cmidrule(lr){2-4} \cmidrule(lr){5-7} \cmidrule(lr){8-10} \cmidrule(lr){11-13} \cmidrule(lr){14-16}
\textbf{Method} & Acc & Len & Red. & Acc & Len & Red. & Acc & Len & Red. & Acc & Len & Red. & Acc & Len & Red. \\
\midrule
\multicolumn{16}{l}{\textit{Qwen3-8B}} \\
Base Model & 95.7 & 4,884 & --- & \textbf{76.2} & 14,229 & --- & \textbf{70.4} & 16,492 & --- & \textbf{61.5} & 7,921 & --- & \textbf{81.9} & 1,633 & --- \\
Concise prompt (v2) & 94.2 & 3,833 & 21.5\% & 74.6 & 12,837 & 9.8\% & 63.7 & 15,622 & 5.3\% & 59.5 & 5,490 & 30.7\% & 82.8 & 1,195 & 26.8\% \\
Concise prompt (v1) & 95.6 & 2,983 & 38.9\% & 74.2 & 11,352 & 20.2\% & 62.1 & 14,194 & 13.9\% & 56.8 & 5,582 & 29.5\% & \textbf{83.0} & 1,195 & 26.8\% \\
\method~(v2) & \textbf{95.7} & 3,339 & 31.6\% & 75.0 & 11,799 & 17.1\% & 65.8 & 13,605 & 17.5\% & 58.3 & 6,558 & 17.2\% & 81.2 & 1,267 & 22.4\% \\
\method~(v1) & \textbf{95.7} & 2,107 & \textbf{56.9\%} & 72.9 & 9,549 & \textbf{32.9\%} & 58.8 & 11,804 & \textbf{28.4\%} & 58.5 & 5,056 & \textbf{36.2\%} & 80.9 & 903 & \textbf{44.7\%} \\
\midrule
\multicolumn{16}{l}{\textit{Qwen3-14B}} \\
Base Model & 93.0 & 4,139 & --- & 75.0 & 13,222 & --- & 69.2 & 15,622 & --- & \textbf{62.2} & 6,185 & --- & \textbf{85.1} & 1,030 & --- \\
Concise prompt (v2) & 94.3 & 3,077 & 25.7\% & 73.3 & 11,509 & 13.0\% & \textbf{71.7} & 14,038 & 10.1\% & 60.5 & 4,579 & 26.0\% & \textbf{84.9} & 801 & 22.2\% \\
Concise prompt (v1) & 95.9 & 2,354 & 43.1\% & \textbf{76.7} & 10,114 & 23.5\% & 66.2 & 12,478 & 20.1\% & 60.6 & 4,569 & 26.1\% & \textbf{84.9} & 805 & 21.8\% \\
\method~(v2) & 95.2 & 2,701 & 34.7\% & 75.0 & 10,615 & 19.7\% & 67.1 & 12,990 & 16.8\% & \textbf{62.0} & 4,905 & 20.7\% & 83.9 & 799 & 22.4\% \\
\method~(v1) & \textbf{96.3} & 1,808 & \textbf{56.3\%} & 73.8 & 8,261 & \textbf{37.5\%} & 62.9 & 10,611 & \textbf{32.1\%} & 61.9 & 3,727 & \textbf{39.7\%} & 84.2 & 586 & \textbf{43.1\%} \\
\midrule
\multicolumn{16}{l}{\textit{DeepSeek-R1-Distill-Llama-8B}} \\
Base Model & 71.3 & 3,313 & --- & 33.3 & 11,042 & --- & 25.0 & 11,983 & --- & 47.0 & 8,609 & --- & 71.5 & 1,923 & --- \\
Concise prompt (v2) & 79.7 & 2,635 & 20.5\% & 42.1 & 10,766 & 2.5\% & 28.8 & 11,522 & 3.8\% & 46.0 & 7,800 & 9.4\% & 73.9 & 1,746 & 9.2\% \\
Concise prompt (v1) & 80.8 & 2,482 & 25.1\% & \textbf{45.0} & 9,920 & \textbf{10.2\%} & \textbf{29.2} & 10,803 & \textbf{9.8\%} & 46.5 & 7,732 & 10.2\% & \textbf{74.1} & 1,746 & 9.2\% \\
\method~(v2) & 79.8 & 2,544 & 23.2\% & 42.1 & 11,317 & $-$2.5\% & 26.2 & 11,975 & 0.1\% & 46.7 & 8,006 & 7.0\% & 71.4 & 1,703 & 11.4\% \\
\method~(v1) & \textbf{82.1} & 2,267 & \textbf{31.6\%} & 39.2 & 10,351 & 6.3\% & 27.1 & 11,131 & 7.1\% & \textbf{48.3} & 7,731 & 10.2\% & 71.7 & 1,585 & \textbf{17.6\%} \\
\bottomrule
\end{tabular}%
}
\end{table*}

\section{Answer-Format Breakdown}
\label{app:answer_format}

Because a grader keyed to a single answer format can undercount accuracy, we score every response three ways (Table~\ref{tab:format}): \emph{answer-only} (an extracted ``Answer: $X$'' line), \emph{boxed-only} (a literal ``$\backslash$boxed\{$\cdot$\}''), and \emph{dual} (correct under either; the scorer used throughout, Section~\ref{sec:setup}). The base model splits correct answers across both formats (MATH-500: 57.7\% answer-only, 54.5\% boxed-only, 93.0\% dual), so no single-format grader captures its accuracy. After \method the model consolidates onto the ``Answer:'' format: answer-only accuracy rises to 92.8\% while boxed-only drops to 10.0\%, yet dual accuracy still improves to 95.2\%. Compression thus removes the redundant post-\texttt{</think>} boxed restatement rather than correctness. The single-format columns sum to more than dual because some responses are correct in both formats; the dual scorer counts this overlap once, i.e.\ $\text{dual}=\text{answer-only}+\text{boxed-only}-\text{both}$. This overlap shrinks from 19.2\% (base) to 7.5\% (\method) on MATH-500 as the formats become disjoint, which explains why a boxed-only grader understates accuracy for models that answer in the ``Answer:'' format.

\begin{table}[t]
\centering
\caption{\textbf{Answer-format breakdown of accuracy (Qwen3-14B, 30K budget, mean@8 \%).} Each response is scored three ways: answer-only (``Answer: $X$''), boxed-only (``$\backslash$boxed\{$\cdot$\}''), and dual (either). \method consolidates correct answers onto the ``Answer:'' format, so answer-only accuracy rises sharply while boxed-only drops, without losing dual accuracy. The dual column matches Table~\ref{tab:main}.}
\label{tab:format}
\small
\begin{tabular}{@{}l l ccc@{}}
\toprule
\textbf{Model} & \textbf{Benchmark} & \textbf{Answer-only} & \textbf{Boxed-only} & \textbf{Dual} \\
\midrule
\multirow{3}{*}{Base Model} & MATH-500 & 57.7 & 54.5 & 93.0 \\
 & AIME 2024 & 20.8 & 70.8 & 75.0 \\
 & AIME 2025 & 17.5 & 65.8 & 69.2 \\
\midrule
\multirow{3}{*}{\method} & MATH-500 & \textbf{92.8} & 10.0 & \textbf{95.2} \\
 & AIME 2024 & \textbf{49.2} & 36.7 & \textbf{75.0} \\
 & AIME 2025 & \textbf{37.5} & 34.6 & 67.1 \\
\bottomrule
\end{tabular}
\end{table}

\section{Training and Implementation Details}
\label{app:implementation}

\paragraph{Technical setup.}
All experiments are conducted on a single node equipped with eight NVIDIA H200 GPUs.
Our implementation is built on top of the \texttt{verl} library~\cite{sheng2024hybridflow}, which provides a HybridEngine for efficient actor--rollout--reference model co-location.
We use PyTorch Fully Sharded Data Parallel (FSDP) for distributed training with parameter and optimizer offloading to CPU, and SGLang~\cite{zheng2024sglang} for batched rollout generation.
Sequence parallelism (Ulysses, degree~4) is enabled during training to handle long sequences efficiently, while tensor parallelism (degree~2) is used for inference.
Mixed-precision training is performed in \texttt{bfloat16}, and gradient checkpointing is enabled to reduce peak memory usage.

\paragraph{Training data.}
Our training data is derived from DAPO-Math-17k~\cite{yu2025dapo}, a deduplicated set of ${\sim}17{,}000$ competition-level math problems.
We randomly split the dataset into 80\% training (${\sim}13{,}600$ prompts) and 20\% validation (${\sim}3{,}400$ prompts) with a fixed seed for reproducibility across all configurations.
For each problem, we construct a \emph{student prompt} (the original question) and a \emph{teacher prompt} (the question prepended with a conciseness instruction).

\paragraph{Training procedure.}
At each training step, the student model generates a response from the student prompt via SGLang sampling (temperature 1.0, top-$p$ 1.0).
We then perform a single gradient update minimizing the reverse KL divergence between student and teacher logit distributions over the student's own generated tokens.
All student rollouts are used for training regardless of correctness; no filtering is applied.
Both teacher and student forward passes are performed for each micro-batch with chunked logit processing (chunk size 256 tokens) to bound peak GPU memory; teacher logits are progressively freed after each chunk.

\paragraph{Hyperparameters.}
Table~\ref{tab:hyperparameters} summarizes the full configuration, which is shared across all models and instruction variants.

\begin{table}[ht]
\renewcommand{\arraystretch}{1.15}
\centering
\setlength{\tabcolsep}{6pt}
\small
\begin{tabular}{ll}
    \toprule
    \textbf{Parameter} & \textbf{Value} \\
    \midrule
    \textbf{General} & \\
    Models & Qwen3-8B, Qwen3-14B, DeepSeek-R1-Distill-Llama-8B \\
    Loss function & Reverse KL: $\KL(\pi_\text{student} \| \pi_\text{teacher})$ \\
    Teacher & Periodic update ($M{=}50$ steps) \\
    \midrule
    \textbf{Data} & \\
    Training prompts & ${\sim}$13,600 (from DAPO-Math-17k) \\
    Validation prompts & ${\sim}$3,400 \\
    Max prompt length & 1,024 tokens \\
    Max response length & 8,192 tokens (training) \\
    \midrule
    \textbf{Generation (student rollout)} & \\
    Inference engine & SGLang \\
    Temperature & 1.0 \\
    Top-$p$ & 1.0 \\
    Rollouts per prompt & 1 \\
    Max generation tokens & 9,216 \\
    \midrule
    \textbf{Evaluation} & \\
    Temperature & 0.6 \\
    Top-$p$ & 0.95 \\
    Top-$k$ & 20 \\
    Rollouts per prompt & 8 \\
    Max generation tokens & 30,000 \\
    Eval frequency & Every 20 steps \\
    \midrule
    \textbf{Training} & \\
    Optimizer & AdamW \\
    Learning rate & $1 \times 10^{-6}$ (constant) \\
    Weight decay & 0.01 \\
    Gradient clipping & 1.0 (max norm) \\
    Global batch size & 32 \\
    Micro-batch size per GPU & 2 \\
    Epochs & 1 \\
    Precision & bfloat16 \\
    \midrule
    \textbf{Infrastructure} & \\
    GPUs & $8 \times$ NVIDIA H200 \\
    Tensor parallelism (inference) & 2 \\
    Sequence parallelism (training) & Ulysses, degree 4 \\
    FSDP parameter offload & Enabled \\
    FSDP optimizer offload & Enabled \\
    Gradient checkpointing & Enabled \\
    \bottomrule
\end{tabular}
\caption{Hyperparameters for \method. The same configuration is used across all models and instruction variants.}
\label{tab:hyperparameters}
\end{table}

\paragraph{Evaluation.}
We evaluate every 20 steps on three held-out math benchmarks: MATH-500~\cite{hendrycks2021math}, AIME~2024, and AIME~2025.
For each benchmark, we generate 8 responses per problem with temperature 0.6, top-$p = 0.95$, and top-$k = 20$, and report mean accuracy (fraction of correct samples averaged over problems) and average response token count.
Correctness is determined by extracting the final answer and comparing against the ground truth using the symbolic and numeric equivalence checker from the \texttt{verl} library~\citep{sheng2024hybridflow}.\footnote{\url{https://github.com/verl-project/verl/blob/main/verl/utils/reward_score/math_dapo.py}}

\section{Alternative Teacher Parameterizations}
\label{app:teacher_variants}

The main paper uses a \emph{periodic teacher update} ($\bar{\theta} \leftarrow \theta$ every $M$ steps) that balances progressive compression with training stability. Here we discuss the design space of teacher parameterizations, from the most conservative to the most aggressive. An empirical comparison across update intervals $M \in \{1, 10, 30, 50, 100\}$ is provided in Section~\ref{sec:ablation_M}.

\paragraph{Frozen teacher ($M = \infty$).} The simplest variant fixes $\bar{\theta} = \theta_0$ for the entire training run. This provides a stable, non-shifting compression target and allows teacher prefill to be pre-computed once. However, the frozen teacher becomes an increasingly weak compression oracle as the student improves, limiting the maximum achievable compression.

\paragraph{Periodic teacher (our default, $M{=}50$).} Setting a finite update interval $M$ (Algorithm~\ref{alg:opsdc}) enables progressive compression: after each refresh, the updated teacher, having already internalized compression from the previous round, produces even more concise traces under instruction $c$, providing a stronger compression signal. The discrete refresh avoids continuous co-adaptation while still allowing compression to deepen over training. Our ablation (Section~\ref{sec:ablation_M}) shows that $M \in \{30, 50\}$ form a stable plateau, retaining base-level accuracy while still compressing substantially, confirming robustness to the exact interval within this range.

\paragraph{EMA teacher.} The teacher parameters are an exponential moving average of the student, updated after each gradient step:
\begin{equation}
    \bar{\theta} \leftarrow \alpha \bar{\theta} + (1-\alpha)\theta, \quad \alpha \in [0.99, 0.999].
\end{equation}
This provides a smooth, continuous version of progressive compression. With moderate decay ($\alpha = 0.995$), it can yield additional compression beyond the frozen teacher but requires careful monitoring for collapse.

\paragraph{Stop-gradient concurrent teacher ($M{=}1$).} The teacher uses the \emph{same} parameters $\theta$ as the student, with stop-gradient during the backward pass. This provides the most aggressive progressive compression but carries the highest risk of \emph{progressive compression collapse}: as the student becomes more concise, the teacher also becomes more concise, creating a positive feedback loop that can drive output length toward degenerate short sequences. Our ablation confirms this prediction: $M{=}1$ causes training collapse, with accuracy falling sharply and response length growing rather than shrinking (Figure~\ref{fig:m_ablation}), consistent with the moving-target instability identified by \citet{shenfeld2025sdft}.

\section{Effect of KL Divergence Direction in \method}
\label{app:kl_direction}

\subsection{Background and Motivation}

The \method distillation loss aligns the live student $p_S$ with the teacher $p_T$,
a periodically frozen snapshot of the student itself.
A natural question is which direction of KL divergence to use:

\begin{itemize}
  \item \textbf{Reverse KL} (our choice): $\mathrm{KL}(p_S \| p_T) =
    \sum_v p_S(v)\bigl[\log p_S(v) - \log p_T(v)\bigr]$.
    The gradient w.r.t.\ student parameters is weighted by the
    \emph{student's own} distribution $p_S(v)$: the student updates only
    in regions where it currently generates, providing built-in
    self-regularization against abrupt distribution shifts.
  \item \textbf{Forward KL} (baseline): $\mathrm{KL}(p_T \| p_S) =
    \sum_v p_T(v)\bigl[\log p_T(v) - \log p_S(v)\bigr]$.
    The gradient is weighted by the \emph{teacher's} distribution $p_T(v)$,
    fully decoupled from the student's current state.
  \item \textbf{Jensen--Shannon divergence} (JSD): the symmetric mixture
    $\tfrac{1}{2}\mathrm{KL}(p_S \| m) + \tfrac{1}{2}\mathrm{KL}(p_T \| m)$
    with $m = \tfrac{1}{2}(p_S + p_T)$. It interpolates between the two
    directions, so its gradient is partly student-weighted and partly
    teacher-weighted.
\end{itemize}

Forward KL is often used in offline distillation, where the teacher is a fixed
external model whose distribution is a reliable target to match. In \method the
teacher is instead a stale copy of the student, refreshed every $M$ steps. We
argue that this makes forward KL ill-suited for \method: because its gradient is
weighted by $p_T$ rather than $p_S$, an update can move probability mass into
regions the current student never visits, with a magnitude independent of how
far the student has drifted from the teacher.

\subsection{Experimental Setup}

We compare three divergences on Qwen3-8B: reverse KL, forward KL, and their
symmetric combination, the Jensen--Shannon divergence (JSD). All three use the
full-vocab setup of Table~\ref{tab:main} with the teacher refreshed every 50
steps, and differ only in the divergence. We track validation accuracy (mean@8)
and mean response length on MATH-500 and AIME 2024 every 20 steps.

\subsection{Results}

Reverse KL and JSD train stably and hold near-base accuracy with strong
compression, whereas forward KL collapses (Table~\ref{tab:kl_direction}): its
accuracy falls to single digits and its response length saturates the
30K-token budget. JSD compresses slightly less than reverse KL, so we adopt
reverse KL as the default.

\begin{table}[h]
\centering
\caption{\textbf{Divergence direction ablation (Qwen3-8B, step 99, 30K budget).} Accuracy (mean@8, \%) and average response length (tokens) on MATH-500 and AIME 2024. Forward KL collapses to degenerate maximum-length generation.}
\label{tab:kl_direction}
\small
\setlength{\tabcolsep}{6pt}
\begin{tabular}{@{}l cc cc@{}}
\toprule
& \multicolumn{2}{c}{\textbf{MATH-500}} & \multicolumn{2}{c}{\textbf{AIME 2024}} \\
\cmidrule(lr){2-3} \cmidrule(lr){4-5}
\textbf{Divergence} & Acc & Len & Acc & Len \\
\midrule
Reverse KL (ours) & \textbf{95.7} & 3{,}339 & \textbf{75.0} & 11{,}799 \\
JSD               & 94.8 & 3{,}281 & 71.7 & 11{,}856 \\
Forward KL        & 5.7 & 29{,}135 & 0.4 & 30{,}000 \\
\bottomrule
\end{tabular}
\end{table}

\begin{figure}[h]
\centering
\includegraphics[width=\linewidth]{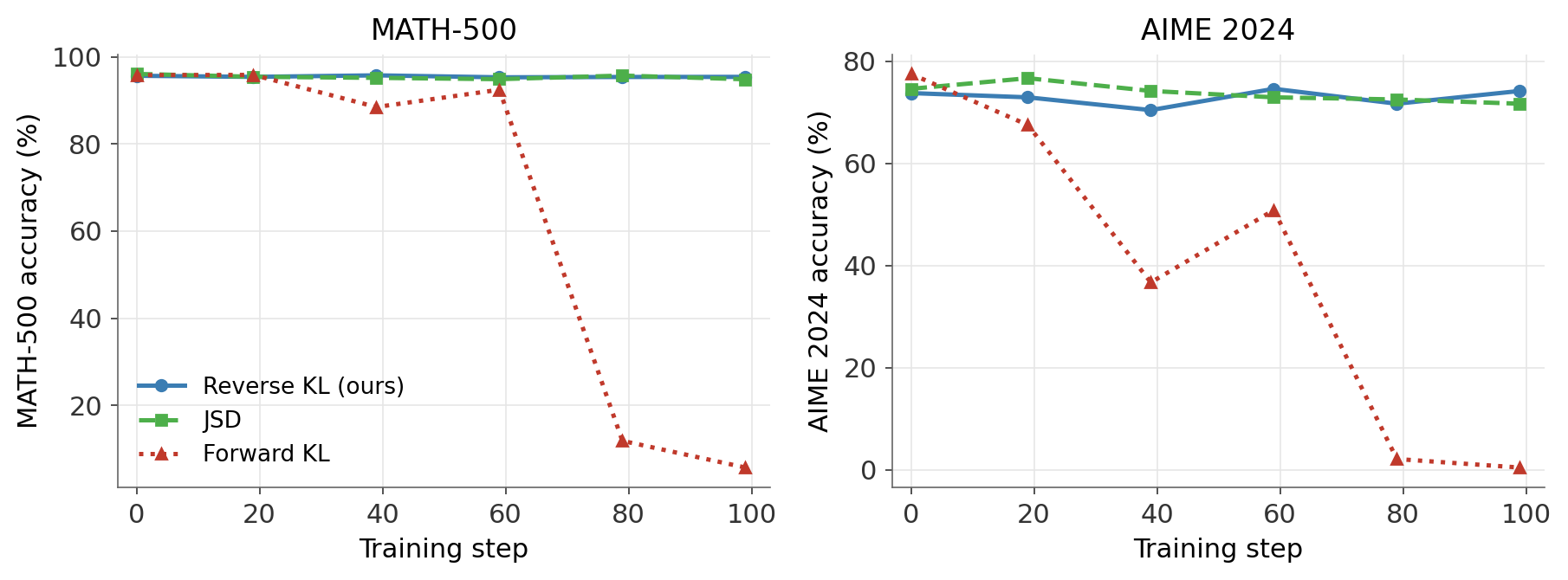}
\caption{Validation accuracy versus training step on Qwen3-8B for reverse KL
  (blue), JSD (green), and forward KL (red), on \textbf{a}, MATH-500 and
  \textbf{b}, AIME 2024. Reverse KL and JSD hold near-base accuracy; forward KL
  collapses to near zero within the first ${\sim}100$ steps.}
\label{fig:kl_comparison}
\end{figure}

\begin{figure}[h]
\centering
\includegraphics[width=\linewidth]{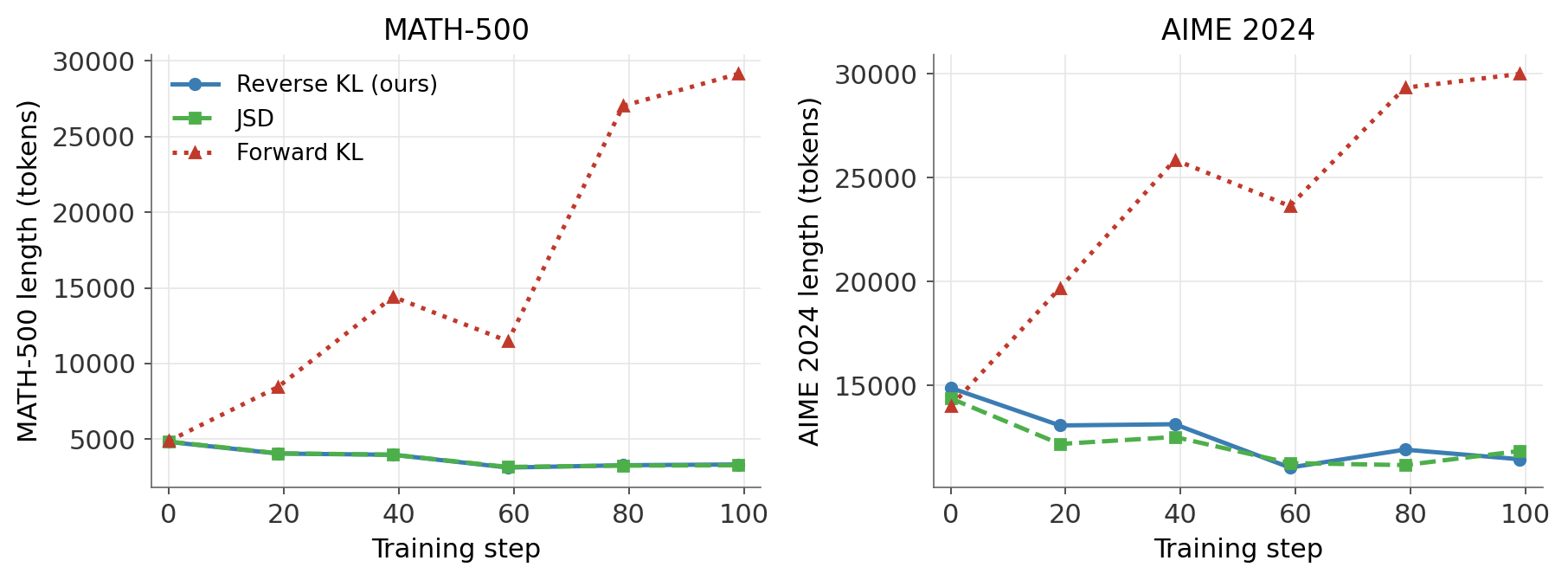}
\caption{Mean response length versus training step on Qwen3-8B for the same
  three divergences, on \textbf{a}, MATH-500 and \textbf{b}, AIME 2024. Reverse
  KL and JSD compress responses smoothly, whereas forward KL diverges to the
  30K-token cap (degenerate maximum-length generation).}
\label{fig:kl_tokens}
\end{figure}

The two failure modes are coupled. Figure~\ref{fig:kl_comparison} shows the
accuracy collapse over the first 100 steps: reverse KL and JSD stay near base
accuracy, while forward KL tracks them briefly and then falls to single digits
on MATH-500 and to near zero on AIME 2024 by step~100. Figure~\ref{fig:kl_tokens}
shows the cause: forward KL does not compress but instead drives response length
up toward the 30K-token budget on both benchmarks, so the policy degenerates
into maximum-length generation, whereas reverse KL and JSD shorten responses
smoothly over the same interval.

\subsection{Mechanism}

The contrast follows from the gradient. Under forward KL,
$\nabla_{\log p_S}\,\mathrm{KL}(p_T\|p_S) = -p_T(v)$, so updates are weighted by
the teacher's distribution regardless of the student's current state. When the
teacher is a stale copy of the student, these updates repeatedly push mass into
regions the current student does not cover, and the drift compounds until the
policy runs to maximum length. Under reverse KL,
$\nabla_\theta\,\mathrm{KL}(p_S\|p_T)=\EE_{v\sim p_S}\!\left[\nabla_\theta\log p_S(v)\,\big(\log p_S(v)-\log p_T(v)+1\big)\right]$,
so updates are weighted by the student's own distribution $p_S$. Because the
student already covers the teacher's high-probability modes, each refresh
requires only a small adjustment, which keeps training stable. JSD inherits
enough of this student-weighted gradient to remain stable, but its forward-KL
component makes it compress slightly less than pure reverse KL. We therefore use
reverse KL for all \method experiments.

\section{Entropy Preservation During Training}
\label{app:entropy}

A recurring failure mode of length-penalized RL is \emph{entropy collapse}: as the policy is pushed
toward shorter outputs it also becomes overconfident, its next-token distribution sharpens, and
generation diversity is lost~\citep{cui2025entropy, wang2025beyond80, xu2026overconfident}. Because
\method also shortens reasoning, a natural concern is whether it induces the same collapse. It does
not.

We track the student's policy entropy directly during training: at every step we compute the mean
full-vocab per-token Shannon entropy of the student's next-token distribution, averaged over the
response tokens of the on-policy rollout (computed with a tensor-parallel entropy reduction over the
sharded vocabulary; this is a logging-only quantity and does not enter the loss). Figure~\ref{fig:entropy}
plots this entropy over the first 99 training steps for Qwen3-8B and Qwen3-14B under the default
recipe (v2 teacher, $M{=}50$, reverse KL).

\begin{figure}[h]
\centering
\includegraphics[width=0.62\linewidth]{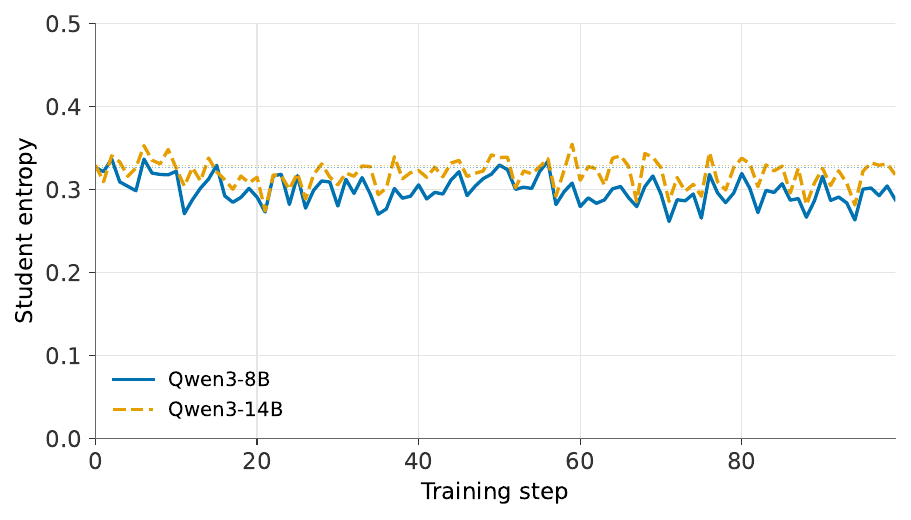}
\caption{\textbf{\method preserves policy entropy while compressing reasoning.} Mean full-vocab
per-token student entropy over training for Qwen3-8B (solid) and Qwen3-14B (dashed); dotted
lines mark each model's step-0 (base) entropy. Entropy fluctuates within a narrow band around its
starting value and shows no downward trend over the run, so the compression gains in
Table~\ref{tab:main} do not come at the cost of a collapsed, overconfident policy.}
\label{fig:entropy}
\end{figure}

For both models the entropy stays within a narrow band around its base-model value for the entire
run: Qwen3-8B moves from $0.33$ at step 0 to $0.29$ at step 99 (range ${\sim}0.27$--$0.34$), and
Qwen3-14B from $0.33$ to $0.32$ (range ${\sim}0.28$--$0.34$). There is no monotone decline toward
zero. This is consistent with the mode-seeking reverse-KL analysis (Appendix~\ref{app:kl_direction})
and the bounded-forgetting result (Proposition~\ref{prop:forgetting}): because the teacher is the
model's own concise mode rather than an external overconfident target, distilling into it removes
redundant tokens without sharpening the retained distribution. \method therefore compresses reasoning
\emph{without} the entropy collapse that length-penalized RL is prone to.

\section{Deep Planning Agentic Task}
\label{app:deepplanning}

Beyond mathematical reasoning, we evaluate whether our compressed models
retain their capability on complex, multi-step agentic tasks that require
tool calling, constraint satisfaction, and long-horizon planning.
We use the \textbf{DeepPlanning} benchmark~\citep{zhang2026deepplanning},
which evaluates LLM agents across two domains:

\paragraph{Travel Planning.}
The agent is given a natural-language travel request (e.g., ``Plan a 3-day
trip to Beijing for two people with a \$2{,}000 budget'') and must produce
a complete itinerary by issuing tool calls to query flights, trains, hotels,
restaurants, and attractions. The benchmark contains 240 test cases
(120 Chinese + 120 English). Evaluation checks both \emph{hard constraints}
(budget, dates, party size) and \emph{commonsense constraints}
(route consistency, business hours, activity diversity) across eight
dimensions. We report the \textbf{composite score}, a weighted average of
commonsense and personalization scores.

\paragraph{Shopping Planning.}
The agent must fulfill a multi-item shopping request (e.g., ``Find running
shoes under \$100, a matching sports watch, and apply any available coupons'')
by navigating a simulated e-commerce environment with tools for product
search, filtering, cart management, and coupon application. The benchmark
contains 150 test cases across three difficulty levels. We report the
\textbf{match rate}, the fraction of expected products correctly placed in
the cart.

\paragraph{Setup.}
We evaluate Qwen3-14B checkpoints from our training (teacher update
frequency $M{=}40$, reverse KL loss) on 100 travel planning samples, evaluating every 10 steps.
Each checkpoint is served via vLLM and paired with the DeepPlanning
function-calling agent, which iteratively invokes the LLM and executes tool
calls until a final plan is produced (up to 400 LLM calls per case).
We run the base model (step~0) 10 times to establish a variance estimate.

\paragraph{Results.}
Figure \ref{fig:deepplanning} shows the trade-off between response length and
planning quality. \method training progressively compresses average LLM
response tokens: shopping tokens decrease from {12.5k} (base) to {6.2k}
(step~120), a \textbf{51\% reduction}; travel tokens decrease from {5.3k}
to {3.1k}, a \textbf{42\% reduction}. Despite this substantial compression,
planning accuracy remains largely preserved: shopping match rate stays
within the baseline variance through step~70 (25.8\% vs.\ 24.2\% base),
and travel composite score is stable across all checkpoints
(17.0--19.1\% vs.\ 17.7\% base). This demonstrates that our reasoning
compression method transfers beyond mathematical reasoning to multi-step agentic
planning, reducing token consumption without degrading the model's ability
to orchestrate complex tool-calling sequences.

\begin{figure}[h]
\centering
\includegraphics[width=\linewidth]{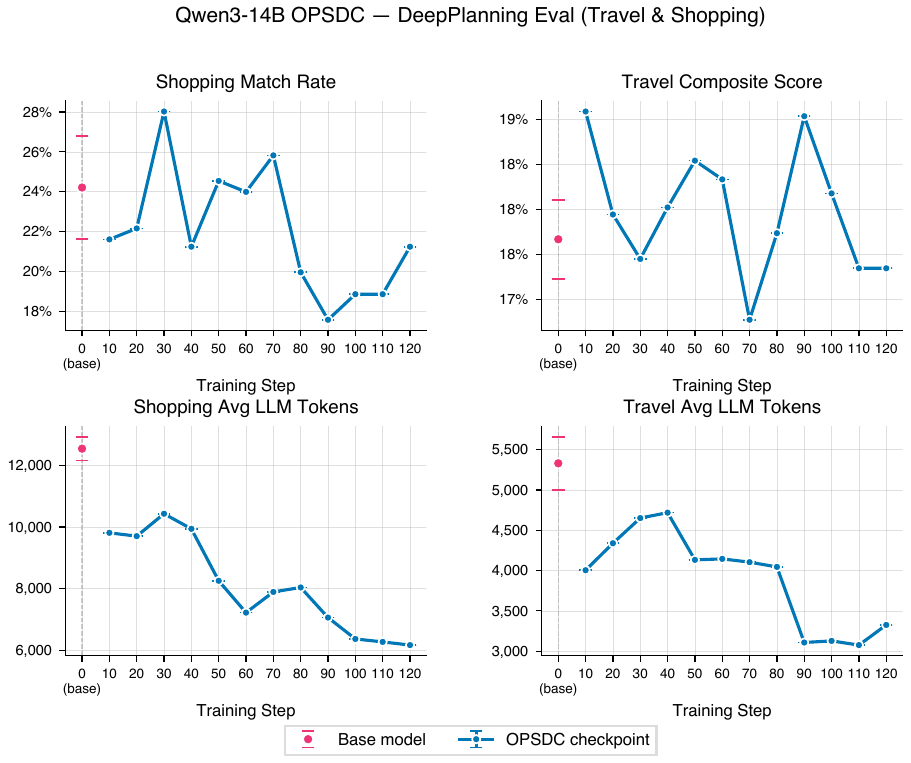}
\caption{Mean response length on training set travel planning and validation set shopping planning. \emph{Shopping match rate} measures the fraction of ground-truth products correctly placed in the agent's cart (higher is better). \emph{Travel composite score} is the average of a commonsense score (route consistency, time feasibility, business hours, etc., across 8 dimensions) and a personalization score (satisfaction of user-specified hard constraints such as budget, dates, and party size), averaged over Chinese and English test sets (higher is better).}
\label{fig:deepplanning}
\end{figure}

\section{Answer-Format Consolidation: Qualitative Examples}
\label{app:qualitative}

The prompt asks for a final line of the form ``Answer: $X$'', but the base Qwen3 models often ignore
this and instead present the final result only inside a ``$\backslash$boxed\{$\cdot$\}'' expression
after \texttt{</think>}, following their post-training convention. A strict ``Answer:''-format grader
scores such responses as wrong even when the boxed value is correct (Appendix~\ref{app:answer_format}).
\method removes this failure mode: the compressed models consolidate onto the requested ``Answer:''
format. The two examples below show the base model ending with a boxed value and no ``Answer:'' line,
while both \method variants (v1 and v2) emit the ``Answer:'' line, in addition to being much shorter.

\begin{figure}[h!]
\centering

\textbf{Problem 1 (MATH-500, intermediate algebra):} \emph{Let $a,b,c$ be real with $|ax^2+bx+c|\le 1$ for all $0\le x\le 1$. Find the largest possible value of $|a|+|b|+|c|$.} \quad (Correct answer: 17)

\vspace{0.4em}
\begin{subfigure}[t]{0.32\textwidth}
\begin{tcolorbox}[enhanced, colback=gray!5, colframe=gray!50, fonttitle=\bfseries\scriptsize, title={Qwen3-8B base (10{,}677 tokens)}, left=3pt, right=3pt, top=2pt, bottom=2pt]
\scriptsize
{\ttfamily\color{gray!60}<think>}

\textcolor{gray}{\ldots}extensive search over the Chebyshev polynomial $8x^2-8x+1$, with repeated verification that no other choice exceeds the bound\ldots

{\ttfamily\color{gray!60}</think>}

\vspace{1pt}\noindent\textcolor{gray!30}{\rule{\linewidth}{0.3pt}}\vspace{2pt}

Thus the maximum of $|a|+|b|+|c|$ is $$\boxed{17}$$

\vspace{2pt}
\textcolor{red}{\scriptsize No ``Answer:'' line; the result appears only inside \texttt{\textbackslash boxed\{\}}.} \rlap{\;\textcolor{red}{\ding{55}}\;\scriptsize\textcolor{red}{(strict format)}}
\end{tcolorbox}
\end{subfigure}
\hfill
\begin{subfigure}[t]{0.32\textwidth}
\begin{tcolorbox}[enhanced, colback=teacherbg, colframe=green!55!black, fonttitle=\bfseries\scriptsize, title={\method~v2 (7{,}893 tokens)}, left=3pt, right=3pt, top=2pt, bottom=2pt]
\scriptsize
{\ttfamily\color{gray!60}<think>}

The polynomial $8x^2-8x+1$ (a scaled Chebyshev polynomial) meets the constraint and maximizes $|a|+|b|+|c|=8+8+1$.

{\ttfamily\color{gray!60}</think>}

\vspace{1pt}\noindent\textcolor{gray!30}{\rule{\linewidth}{0.3pt}}\vspace{2pt}

The maximum is $8+8+1$.

Answer: \textcolor{green!50!black}{$\boldsymbol{17}$} \rlap{\;\textcolor{green!50!black}{\checkmark}}
\end{tcolorbox}
\end{subfigure}
\hfill
\begin{subfigure}[t]{0.32\textwidth}
\begin{tcolorbox}[enhanced, colback=orange!3, colframe=orange!60!black, fonttitle=\bfseries\scriptsize, title={\method~v1 (4{,}144 tokens)}, left=3pt, right=3pt, top=2pt, bottom=2pt]
\scriptsize
{\ttfamily\color{gray!60}<think>}

Chebyshev choice $8x^2-8x+1$ satisfies $|ax^2+bx+c|\le1$ on $[0,1]$, giving $|a|+|b|+|c|=8+8+1=17$.

{\ttfamily\color{gray!60}</think>}

\vspace{1pt}\noindent\textcolor{gray!30}{\rule{\linewidth}{0.3pt}}\vspace{2pt}

$|a|+|b|+|c| = 8+8+1 = 17$.

Answer: \textcolor{green!50!black}{$\boldsymbol{17}$} \rlap{\;\textcolor{green!50!black}{\checkmark}}
\end{tcolorbox}
\end{subfigure}

\vspace{1.2em}

\textbf{Problem 2 (AIME 2024, geometry):} \emph{Tetrahedron $ABCD$ has $AB{=}CD{=}\sqrt{41}$, $AC{=}BD{=}\sqrt{80}$, $BC{=}AD{=}\sqrt{89}$. The equal face-distance from the interior point can be written $\tfrac{m\sqrt n}{p}$; find $m{+}n{+}p$.} \quad (Correct answer: 104)

\vspace{0.4em}
\begin{subfigure}[t]{0.32\textwidth}
\begin{tcolorbox}[enhanced, colback=gray!5, colframe=gray!50, fonttitle=\bfseries\scriptsize, title={Qwen3-8B base (7{,}971 tokens)}, left=3pt, right=3pt, top=2pt, bottom=2pt]
\scriptsize
{\ttfamily\color{gray!60}<think>}

\textcolor{gray}{\ldots}computes the volume and total face area, finds inradius $\tfrac{20\sqrt{21}}{63}$, then $m{=}20,n{=}21,p{=}63$\ldots

{\ttfamily\color{gray!60}</think>}

\vspace{1pt}\noindent\textcolor{gray!30}{\rule{\linewidth}{0.3pt}}\vspace{2pt}

Thus $m+n+p = 20+21+63 = \boxed{104}$

\vspace{2pt}
\textcolor{red}{\scriptsize No ``Answer:'' line; result only inside \texttt{\textbackslash boxed\{\}}.} \rlap{\;\textcolor{red}{\ding{55}}\;\scriptsize\textcolor{red}{(strict format)}}
\end{tcolorbox}
\end{subfigure}
\hfill
\begin{subfigure}[t]{0.32\textwidth}
\begin{tcolorbox}[enhanced, colback=teacherbg, colframe=green!55!black, fonttitle=\bfseries\scriptsize, title={\method~v2 (5{,}854 tokens)}, left=3pt, right=3pt, top=2pt, bottom=2pt]
\scriptsize
{\ttfamily\color{gray!60}<think>}

Inradius $r=\tfrac{20\sqrt{21}}{63}$, so $m{=}20,n{=}21,p{=}63$ and $m+n+p=104$.

{\ttfamily\color{gray!60}</think>}

\vspace{1pt}\noindent\textcolor{gray!30}{\rule{\linewidth}{0.3pt}}\vspace{2pt}

$m+n+p = 20+21+63 = 104$.

Answer: \textcolor{green!50!black}{$\boldsymbol{104}$} \rlap{\;\textcolor{green!50!black}{\checkmark}}
\end{tcolorbox}
\end{subfigure}
\hfill
\begin{subfigure}[t]{0.32\textwidth}
\begin{tcolorbox}[enhanced, colback=orange!3, colframe=orange!60!black, fonttitle=\bfseries\scriptsize, title={\method~v1 (5{,}833 tokens)}, left=3pt, right=3pt, top=2pt, bottom=2pt]
\scriptsize
{\ttfamily\color{gray!60}<think>}

$r=\tfrac{20\sqrt{21}}{63}$ in the form $\tfrac{m\sqrt n}{p}$ with $m{=}20,n{=}21,p{=}63$.

{\ttfamily\color{gray!60}</think>}

\vspace{1pt}\noindent\textcolor{gray!30}{\rule{\linewidth}{0.3pt}}\vspace{2pt}

$m+n+p=104$.

Answer: \textcolor{green!50!black}{$\boldsymbol{104}$} \rlap{\;\textcolor{green!50!black}{\checkmark}}
\end{tcolorbox}
\end{subfigure}

\caption{\textbf{\method consolidates onto the requested ``Answer:'' format.} On a MATH-500 problem (top) and an AIME 2024 problem (bottom), the base Qwen3-8B model presents its final result only inside ``$\backslash$boxed\{$\cdot$\}'' with no ``Answer:'' line, which a strict answer-format grader scores as wrong even though the boxed value is correct. Both \method variants (v1 uniform, v2 difficulty-aware) emit the requested ``Answer:'' line and are also substantially shorter. This is the mechanism behind the answer-only/boxed-only shift in Appendix~\ref{app:answer_format}.}
\label{fig:qualitative}
\end{figure}

\end{document}

%% file: extra_pkgs.tex
\usepackage[export]{adjustbox}
\usepackage{float}
\usepackage[ruled]{algorithm2e}
\usepackage[inline, shortlabels]{enumitem}
\usepackage[T1]{fontenc}
\usepackage[colorlinks=true,linkcolor=blue!70!black,citecolor=blue!70!black,urlcolor=blue!70!black]{hyperref}
\usepackage{microtype}
\usepackage{pifont}
\newcommand{\cmark}{\ding{51}}
\newcommand{\xmark}{\ding{55}}
\usepackage{xcolor}
\usepackage{xurl}
\usepackage[most]{tcolorbox}
\definecolor{studentbg}{HTML}{FFF3E0}
\definecolor{teacherbg}{HTML}{E8F5E9}
\usepackage{graphicx}
\usepackage{booktabs}
\usepackage{makecell}
\usepackage{tabularray}
\usepackage{tabularx}
\usepackage{listings}
\usepackage{amsmath, amsfonts}
\usepackage{nicefrac}

\UseTblrLibrary{booktabs}

\makeatletter
\newcommand{\ssymbol}[1]{\@fnsymbol{#1}}
\newcommand{\romanNumeral}[1]{\expandafter\@slowromancap\romannumeral #1@}
\makeatother